\documentclass[10pt,twocolumn]{article} 

\usepackage{simpleConference}
\usepackage{times}
\usepackage{graphicx}
\usepackage{amssymb}
\usepackage{amsmath}
\PassOptionsToPackage{hyphens}{url}\usepackage{hyperref}
\usepackage[numbers]{natbib}
\usepackage{subcaption}
\usepackage{amsthm}
\usepackage{amsbsy}
\usepackage{multirow}
\usepackage{booktabs}
\usepackage{authblk}
\usepackage{color}
\usepackage{float}

\newtheorem{theorem}{Theorem}

\title{To Trust or Not To Trust Prediction Scores for Membership Inference Attacks}

\makeatletter
\newcommand{\printfnsymbol}[1]{%
  \textsuperscript{\@fnsymbol{#1}}%
}
\makeatother

\author[1]{Dominik Hintersdorf\thanks{equal contribution}}
\author[1]{Lukas Struppek\printfnsymbol{1}}
\author[1,2]{Kristian Kersting}
\affil[1]{Department of Computer Science, Technical University of Darmstadt, Germany}
\affil[2]{Centre for Cognitive Science, TU Darmstadt, and Hessian Center for AI (hessian.AI), Germany}
\affil[ ]{\{dominik.hintersdorf, lukas.struppek, kersting\}@cs.tu-darmstadt.de}

\begin{document}

\maketitle
\thispagestyle{empty}

\begin{abstract}
Membership inference attacks (MIAs) aim to determine whether a specific sample was used to train a predictive model. Knowing this may indeed lead to a privacy breach. Most MIAs, however, make use of the model's prediction scores---the probability of each output given some input---following the intuition that the trained model tends to behave differently on its training data. We argue that this is a fallacy for many modern deep network architectures. Consequently, MIAs will miserably fail since overconfidence leads to high false-positive rates not only on known domains but also on out-of-distribution data and implicitly acts as a defense against MIAs. Specifically, using generative adversarial networks, we are able to produce a potentially infinite number of samples falsely classified as part of the training data. In other words, the threat of MIAs is overestimated, and less information is leaked than previously assumed. Moreover, there is actually a trade-off between the overconfidence of models and their susceptibility to MIAs: the more classifiers know when they do not know, making low confidence predictions, the more they reveal the training data.\footnote{Published as a conference paper at IJCAI-ECAI 2022.}
\end{abstract}

\section{Introduction}\label{sec:introduction}

Deep learning models achieve state-of-the-art performances in various tasks such as computer vision, language modeling, and healthcare. However, large datasets are needed to train these models. Collecting and, in particular, cleaning and labeling data is expensive. Hence, users may look for alternative data sources, which may not always be legal ones.
To detect data abuse, it would be desirable to prove whether a model was trained on leaked or unauthorized retrieved data. However, to prove that a specific data point was part of the training set is difficult since neural networks do not store plain training data like lazy learners. Instead, the learned knowledge is encoded into the network's weights.

One way to distinguish between unseen data and data points used for training the neural networks is through membership inference attacks (MIAs). They attempt to identify training samples in a large set of possible inputs. Besides malicious intentions, MIAs might be used to prove illegal data abuse in deep learning settings. To use membership inference results as evidence in court, high accuracy and robustness to different data types and network architectures is required.
\begin{figure}[t]
\centering
\includegraphics[width=.83\linewidth]{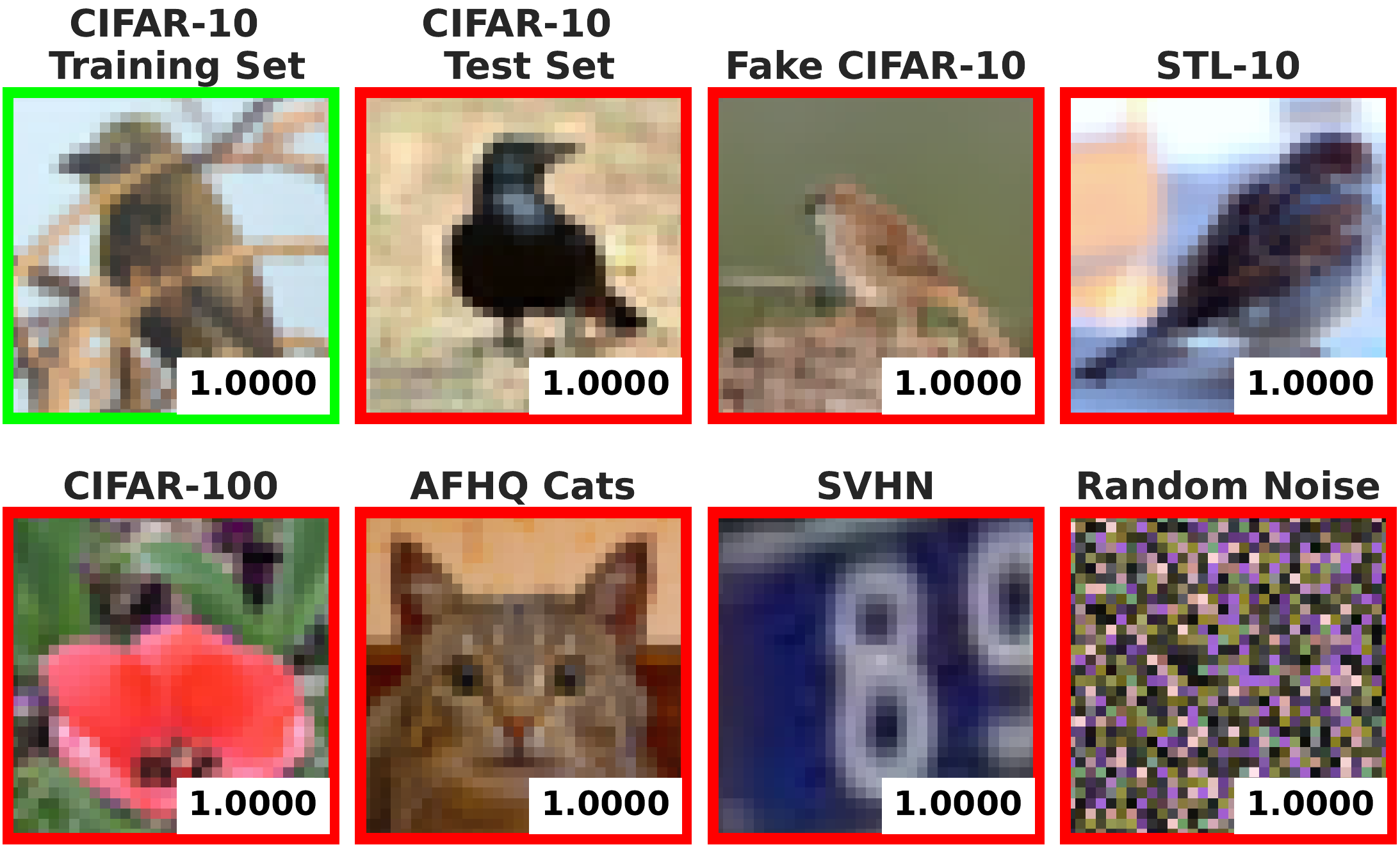}
\caption{False-positive membership inference attacks (red frames) against a ResNet-18 and their assigned maximum prediction scores.
\label{fig:false_positive_samples}}
\vskip -0.3in
\end{figure}

Previous works on MIAs, see e.g., ~\citet{shokri_mi2016} and \citet{salem2018mlleaks}, state strong attack results in distinguishing between training and test data, and give the impression that MIAs have a strong impact on a model's privacy. However, the evaluation of MIAs reported in the literature is usually done with limited data in a cross-validation setting, i.e., on samples from the exact same data distribution, not considering other distributions with possibly similar image contents.
 
We argue that MIAs, in particular attacks based on a model's prediction scores, are not robust and not very meaningful in realistic settings, due to their high false-positive rates, also criticized by \citet{rezaei2021difficulty}. We take, however, a broader view and do not restrict evaluation on the target model's exact training distribution. In a specific domain, there is a possibly infinite number of samples and hence the number of false positives can be increased arbitrarily. This leads to reduced informative value and low reliability of the attacks under realistic conditions. Fig.~\ref{fig:false_positive_samples} shows samples from various datasets for which all three MIAs studied in this paper make false-positive predictions, even if the inputs are nothing similar to the training data or do not contain any meaningful information at all. We practically demonstrate the theoretically unlimited number of false-positive member classifications by using a GAN to generate images following the training distribution. 

Our argumentation is based on the already known overconfidence of modern deep neural architectures~\citep{nguyen2014easyfool,hendrycks2016baseline,guo2017calibration,leibig2017sr}. However, overconfidence has consistently been ignored in the MIA literature, even though MIA findings are already having an impact on regulatory and other legal measures. Our experimental results indicate that mitigating the overconfidence of neural networks using calibration techniques increases privacy leakage. 

We argue that previous works performed misleading attack evaluations and overestimated the actual attack effectiveness by using only data from the target model's exact training distribution. Actually, there might not exist any meaningful MIA at all since the attacks will always produce a high number of false positives due to the overconfidence of neural networks.

To summarize, we make the following contributions:
\begin{enumerate}
\item We demonstrate that the effectiveness of MIAs has been systematically overestimated by ignoring the fact that most neural networks are inherently overconfident and, therefore, produce high false-positive rates.
\item We show that overconfidence acts as a natural defense against MIAs.
\item We reveal that a trade-off exists between keeping models secure against MIAs and mitigating overconfidence.
\end{enumerate}
We proceed as follows. We start off by reviewing MIAs and how overconfidence of neural networks can be mitigated. Afterward, we introduce the theoretical background and our experimental setup. Before discussing and concluding our work, we present our experimental results. 

\begin{figure}[t]
\centering
 \includegraphics[width=.83\linewidth]{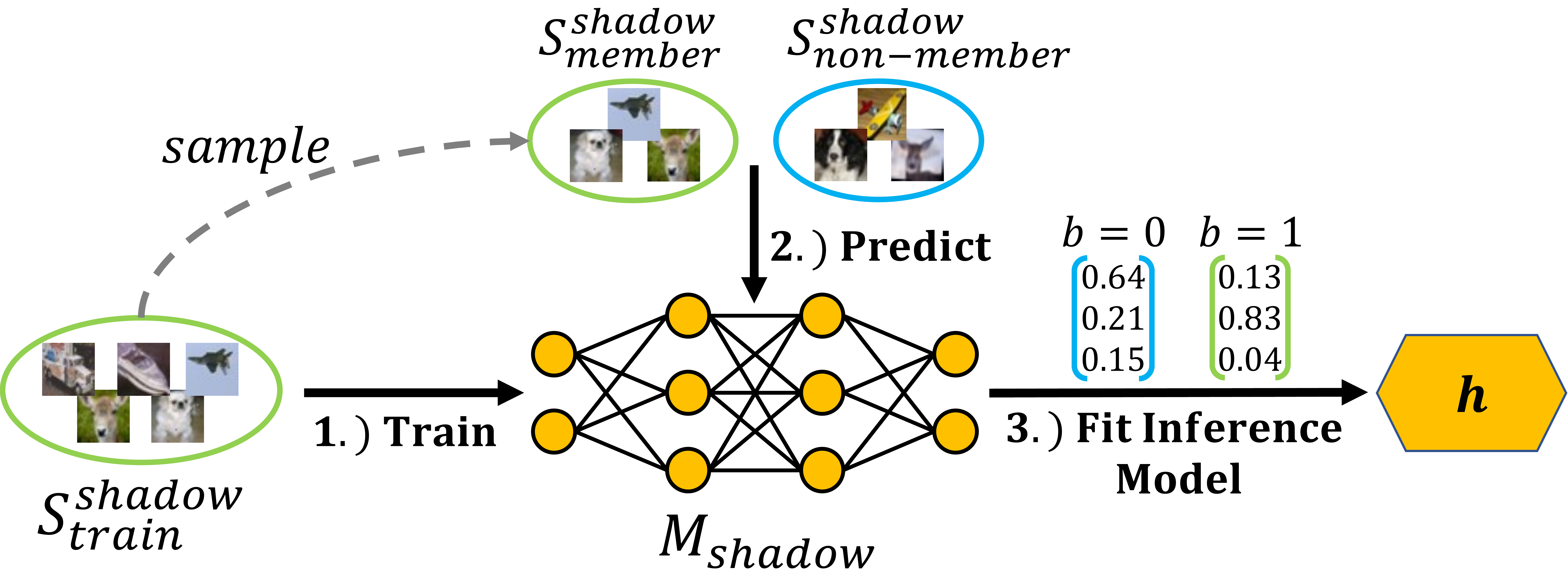}
\caption{Membership inference preparation process. 
\label{fig:shadow_training}}
\vskip -0.1in
\end{figure}

\begin{figure}[t]
    \centering
    \includegraphics[width=\linewidth]{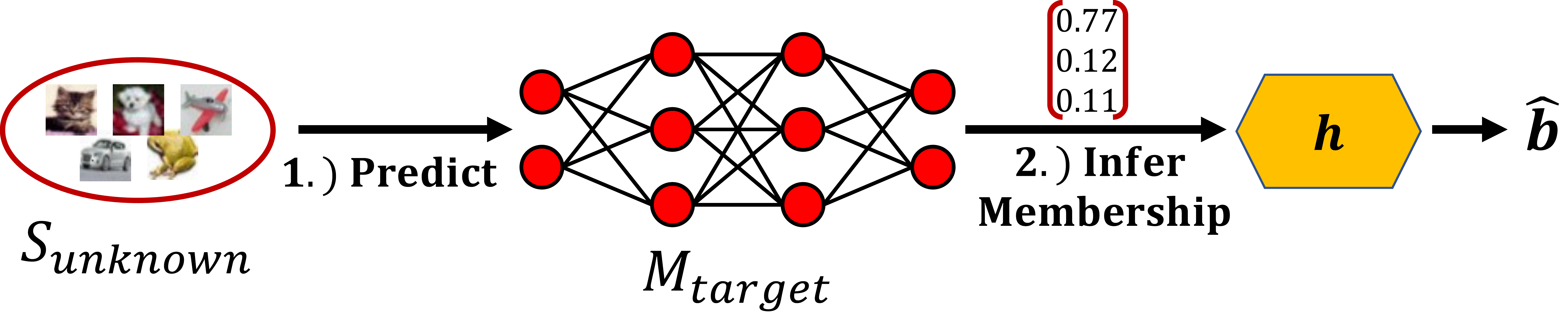}
    \caption{Application of inference model $h$ against a target model $M_{target}$. 
    The adversary first queries $M_{target}$ to collect prediction scores, and the inference model $h$ is then used to make a prediction $\hat{b}$ on the membership status.
    \label{fig:membership inference}}
    \vskip -0.1in
\end{figure}

\section{Membership Inference Attacks}\label{sec:mi_attacks}
Membership inference attacks (MIAs) on neural networks were first introduced by \citet{shokri_mi2016}. In a general MIA setting, as usually assumed in the literature, an adversary is given an input $x$ following distribution $D$ and a target model $M_{target}$ which was trained on a training set $S_{train}^{target} \sim D^n$ with size $n$. The adversary is then facing the problem to identify whether a given $x \sim D$ was part of the training set $S_{train}^{target}$. To predict the membership of $x$, the adversary creates an inference model $h$. In score-based MIAs, the input to $h$ is the prediction score vector produced by $M_{target}$ on sample $x$. Since MIAs are binary classification problems, precision, recall, false-positive rate (FPR), and area under the receiver operating characteristic (AUROC) are used as attack evaluation metrics in our experiments.

All MIAs exploit a difference in the behavior of $M_{target}$ on seen and unseen data. Most attacks in the literature follow \citet{shokri_mi2016} and train so-called shadow models $M_{shadow}$ on a disjoint dataset $S_{train}^{shadow}$ drawn from the same distribution $D$ as $S_{train}^{target}$. $M_{shadow}$ is used to mimic the behavior of $M_{target}$ and adjust parameters of $h$, such as threshold values or model weights. Note that the membership status for inputs to $M_{shadow}$ are known to the adversary. Fig.~\ref{fig:shadow_training} visualizes the attack preparation process. 

In recent years, various MIAs have been proposed. \citet{shokri_mi2016} trained multiple shadow models and queried each of the shadow models with its training data (members), as well as unseen data (non-members) to retrieve the prediction scores of the shadow models. Multiple binary classifiers were then trained for each class label to predict the membership status. \citet{salem2018mlleaks} also used prediction scores and trained a single class-agnostic neural network to infer membership. In contrast to \citet{shokri_mi2016}, their approach relies on a single shadow model. The input of $h$ consists of the $k$ highest prediction scores in descending order. 

Instead of focusing solely on the scores, \citet{Yeom_Privacy_Risk} took advantage of the fact that the loss of a model is lower on members than on non-members and fit a threshold to the loss values. More recent approaches \citep{choquettechoo2021labelonly,li2021membership} focused on label-only attacks where only the predicted label for a known input is observed. 

\section{Overconfidence of Neural Networks}\label{sec:overconfidence}
Neural networks usually output prediction scores, e.g., by applying a softmax function. To take model uncertainty into account, it is usually desired that the prediction scores represent the probability of a correct prediction, which is usually not the case. This problem is generally referred to as model calibration. \citet{guo2017calibration} demonstrated that modern networks tend to be overconfident in their predictions. \citet{Hein2019WhyRN} have further proven that ReLU networks are overconfident even on samples far away from the training data.

Existing approaches to mitigate overconfidence can be grouped into two categories: post-processing methods applied on top of trained models and regularization methods modifying the training process. As a post-processing method, \citet{guo2017calibration} proposed temperature scaling using a single temperature parameter $T$ for scaling down the pre-softmax logits for all classes. The larger $T$ is, the more the resulting scores approach a uniform distribution. \citet{kristiadi2020being} further proposed to approximate a model's final layer with a Laplace approximation. \citet{mueller2020does} demonstrated that label smoothing regularization \citep{szegedy2015rethinking} not only improves the generalization of a model but also implicitly leads to better model calibration. The calibration of a model can be measured by the expected calibration error (ECE) \citep{Naeini2015} and the overconfidence error (OE) \citep{overconfidence_error}. Both metrics compute a weighted average over the absolute difference between test accuracy and prediction scores while ECE penalizes the calibration gap and OE penalizes overconfidence.

\section{Do Not Trust Prediction Scores for MIAs}\label{sec:dont_trust_mia}
In this section, we will show that predictions scores for MIAs cannot be trusted because score-based MIAs make membership decisions based mainly on the maximum prediction score. As a first step, we mathematically motivate our argumentation and then verify our claims empirically.

Formally, a neural network $f(x)$ using ReLU activations decomposes the unrestricted input space $\mathbb{R}^m$ into a finite set of polytopes (linear regions). We can then interpret $f(x)$ as a piecewise affine function that is affine in any polytope~\citep{arora2018understanding}.
Due to the limited number of polytopes, the outer polytopes extend to infinity which allows to arbitrarily increase the prediction scores through scaling inputs by a large constant $\delta$~\citep{Hein2019WhyRN}. We now further develop these findings from an MIA point of view and state the following theorem:

\begin{theorem}\label{prop:high_fpr} 
Given a (leaky) ReLU-classifier, we can force almost any non-member input to be classified as a member by score-based MIAs, simply by scaling it by a large constant.
\end{theorem}
\begin{proof}
Let~${f:\mathbb{R}^m \rightarrow \mathbb{R}^d}$ be a piecewise affine (leaky) ReLU-classifier. We define a score-based MIA inference model ${h: \mathbb{R}^d \rightarrow \{0, 1\}}$ with $1$ indicating a classification as a member. For almost any input $x\in \mathbb{R}^m$ and a sufficiently small $\epsilon>0$ if $\max_{i=1,...,d} \,\, f(x)_i \geq 1-\epsilon$, it follows that $h(f(x)) = 1$. Since ${lim_{\delta \rightarrow \infty} \max_{i=1,...,d} f(\delta x)_i = 1}$, 
then $lim_{\delta \rightarrow \infty} h(f(\delta x)) {= 1}$ already holds. 
\end{proof}

See Appx.~\ref{app:proof} for an extended proof. 
By scaling the whole non-member dataset, one can force the FPR to be close to 100\%. Indeed, the theorem holds only for (leaky) ReLU-networks and unbounded inputs. However, since uncalibrated neural networks assign high prediction scores to a wide range of different inputs, the number of false-positive predictions is also large for unscaled inputs from known and unknown domains. Next, we empirically show that one cannot trust predictions scores for MIAs in more general settings without input scaling required and using other activation functions. \\ 

\subsection{Experimental Protocol}\label{sec:experiments}
We make our source code publicly available\footnote{Available at \href{https://github.com/ml-research/To-Trust-or-Not-To-Trust-Prediction-Scores-for-Membership-Inference-Attacks}{https://github.com/ml-research/To-Trust-or-Not-To-Trust-Prediction-Scores-for-Membership-Inference-Attacks}.} and provide further information for reproducibility in Appx.~\ref{app:additional_experimental_results}.

{\bf Threat Model.}\label{sec:threat_model}
As in most MIA literature~\citep{salem2018mlleaks,Yeom_Privacy_Risk,song_systematic_evaluation}, we followed the MIA setting of \citet{shokri_mi2016} and like \citet{salem2018mlleaks} only trained a single shadow model for each attack. As in previous work, we also simulate a worst-case scenario, i.e., the adversary knows the exact architecture and training procedure of the target model. Therefore, a strong shadow model can be trained, following the procedure depicted in Fig.~\ref{fig:shadow_training}. In our score-based MIA scenario, the adversary only has access to the target model's prediction scores.

{\bf Datasets.} We evaluated the attacks on models trained on the CIFAR-10~\citep{Krizhevsky09learningmultiple} and Stanford Dogs \citep{KhoslaYaoJayadevaprakashFeiFei_FGVC2011} datasets. 

We created two disjoint training datasets for the target and shadow models, each containing 12,500 (CIFAR-10) and 8,232 (Stanford Dogs) samples. We then randomly drew 2,500 and 2,058 samples, respectively, from the training and test sets to create the member and non-member datasets. 

We used various datasets to demonstrate the susceptibility of prediction score-based MIAs to high scores on samples from neighboring distributions and samples further away from the training data---a kind of out-of-distribution (OOD) setting. We used STL-10~\citep{pmlr-v15-coates11a}, CIFAR-100~\citep{Krizhevsky09learningmultiple}, SVHN~\citep{svhn}, and Animal Faces-HQ (AFHQ)~\citep{choi2020starganv2} as datasets.

Additionally, we used pre-trained StyleGAN2~\citep{karras2020training} models to generate synthetic CIFAR-10 and dog images, referred to as Fake CIFAR-10 and Fake Dogs. To empirically verify our theorem and push our approach to the extreme, we created two additional datasets based on the respective test images by scaling pixel values by factor 255 and randomly permuting the images' pixels to create random noise samples. In the following, we refer to these two datasets as Permuted and Scaled.

{\bf Neural Network Architectures.}\label{sec:model_architectures} On CIFAR-10, we trained a ResNet-18 \citep{resnet}, an EfficientNetB0 \citep{tan2020efficientnet} and a simple convolutional neural network following \citep{salem2018mlleaks}, referred to as SalemCNN.
For the Stanford Dogs dataset, we used a larger ResNet-50 architecture pre-trained on ImageNet. ResNets and SalemCNN are ReLU networks and can be interpreted as piecewise linear functions~\citep{arora2018understanding}.
EfficientNetB0 uses Swish activation functions \citep{ramachandran2017searching}, which are not piecewise linear and, therefore, our theorem does not hold. Nevertheless, we demonstrate that also non-ReLU networks suffer from overconfidence, leading to weak MIAs.

\begin{table}[t!]
    \centering
    \resizebox{\columnwidth}{!}{
    \begin{tabular}{lccc}
    \toprule
                                & \textbf{SalemCNN}     & \textbf{ResNet-18}    & \textbf{EfficientNetB0}  \\
    \midrule 
    Train Accuracy              & 100.00\%              & 100.00\%              & 99.03\%   \\
    Test Accuracy               & 59.04\%               & 69.38\%               & 71.06\%   \\
    \midrule 
    Entropy Pre                 & 65.51\%               & 67.35\%               & 61.36\%   \\
    Entropy Rec                 & 88.52\%               & 92.32\%               & 79.96\%   \\
    Entropy FPR                 & 46.60\%               & 44.76\%               & 50.36\%   \\
    Entropy AUROC               & 70.94\%               & 76.50\%               & 66.57\%   \\
    \midrule 
    Max. Score Pre              & 65.34\%               & 67.35\%               & 61.43\%   \\
    Max. Score Rec              & 87.48\%               & 92.32\%               & 79.64\%   \\
    Max. Score FPR              & 46.40\%               & 44.76\%               & 50.00\%   \\
    Max. Score AUROC            & 72.03\%               & 77.50\%               & 66.58\%   \\
    \midrule 
    Top-3 Scores Pre            & 62.48\%               & 63.84\%               & 60.74\%   \\
    Top-3 Scores Rec            & 100.00\%              & 98.04\%               & 82.60\%   \\
    Top-3 Scores FPR            & 60.04\%               & 55.52\%               & 53.40\%   \\ 
    Top-3 Scores AUROC          & 71.57\%               & 77.14\%               & 66.61\%   \\
    \bottomrule
    \end{tabular}
        }
    \caption{Training and attack metrics for the target models trained on CIFAR-10. We measure the attacks' precision (Pre), recall (Rec), FPR and AUROC  on equally-sized member and non-member subsets from CIFAR-10.}
    \label{tab:cifar10_model_results}
\end{table}

{\bf Prediction Score-Based Attacks. }\label{sec:confidence_mi_attacks}
We base our analysis on three different MIAs~\citep{salem2018mlleaks} exploiting the top-3 values of the prediction score vector, the maximum value, and the entropy. For the top-3 prediction score attack, we trained a small neural network with a single hidden layer as an inference model. It uses the three highest scores of $M_{target}$ in descending order as inputs. The maximum prediction score attack relies only on the highest score, while the entropy attack computes the entropy on the whole prediction score vector. An input sample is classified as a member, if the maximum value is higher or if the entropy is lower than a threshold. We fit all attack models on the shadow models' outputs, with the thresholds chosen to maximize the true-positive rate while minimizing the FPR. 

\subsection{Experimental Results}\label{sec:empirical_results}
We investigate the following questions: {\bf (Q1)} How robust are prediction score-based MIAs? {\bf (Q2)} Does overconfidence negatively affect MIAs? {\bf (Q3)} How does calibrating neural networks influence the success of MIAs? {\bf (Q4)} Are defenses contrary to calibration?  

{\bf (Q1) MIAs Are Not Robust.}\label{sec:mi_not_robust}
Tab.~\ref{tab:cifar10_model_results} summarize the test accuracy and attack metrics of the CIFAR-10 target models. The different attacks performed quite similarly while the recall is always significantly higher than the precision, indicating the problem of many false-positive predictions. A similar picture emerges when looking at the results of the Standard Stanford Dog model, stated in Tab.~\ref{tab:model_attack_metrics_dogs}. We state additional threshold-free metrics, including AUPR and FPR@95\%TPR, in Appx.~\ref{app:add_metrics}.

To examine the robustness of the attacks, we used the remaining datasets as non-member inputs and measured the FPRs. Figs.~\ref{fig:fpr_resnet18_cifar10_ls} and~\ref{fig:fpr_resnet18_cifar10_llla} (transparent bars), show the FPR of the attacks against the ResNet CIFAR-10 models, and Figs.~\ref{fig:fpr_resnet50_dogs_ls} and ~\ref{fig:fpr_resnet50_dogs_la} do the same for the Stanford Dogs models. We state visualizations for the other CIFAR-10 models, which behave similar to the ResNet-18 model, in Appx.~\ref{app:cifar10_fpr_results}.

The results demonstrate that the attacks not only tend to falsely classify samples from the test data as members but also samples from other distributions. For example, the attacks against CIFAR-10 misclassified more than a third of the STL-10 samples, which are similar in content and style, as members. The same holds for AFHQ Dogs samples as input for the Stanford Dogs model. The results on the remaining datasets, especially on the scaled samples, empirically confirm our theorem and demonstrate that neural networks are not able to recognize when they are operating on unknown inputs, such as housing numbers, cats, or random noise, and therefore still produce high FPRs. Even on generated Fake samples following the training distribution, the FPR is comparably high and shows that there exists a potentially infinite number of false-positive samples that are not out-of-distribution. This behavior is not limited to ReLU networks. The FPR of the EfficientNetB0 on the datasets is quite similar to the FPR of the ResNet-18. This indicates that the problem of high FPR in MIAs is affecting modern deep architectures in general and underlines the fact that MIAs are not robust.

\begin{table}[t!]
    \centering
    \resizebox{\columnwidth}{!}{
    \begin{tabular}{lc|cc|cc}
    \toprule
                            & \multicolumn{1}{c}{}  & \multicolumn{2}{c}{\textbf{Calibration}}              & \multicolumn{2}{c}{\textbf{Defenses}} \\
    \textbf{ResNet-50}      & \textbf{Standard}     & \textbf{LS}                           & \textbf{LA}   & \textbf{Temp}                         & \textbf{L2}           \\
    \midrule
    Train Accuracy          & 98.48\%               & 99.62\%                               & 98.48\%       & 98.48\%                               & 74.05\%                           \\
    Test Accuracy           & 59.69\%               & 64.65\%                               & 59.62\%       & 59.69\%                               & 48.15\%                           \\
    ECE                     & \textbf{25.09\%}      & $\pmb{\downarrow}$\textbf{5.80\%}     & 5.63\%        & 51.03\%                               & 11.86\%                           \\
    OE                      & \textbf{21.18\%}      & $\pmb{\downarrow}$\textbf{0.32\%}     & 3.59\%        &  0.0\%                                & 7.83\%                            \\
    \midrule
    Entropy Pre             & 68.22\%               & 76.33\%                               & 65.39\%       & 59.45\%                               & 60.50\%                           \\
    Entropy Rec             & 84.50\%               & 82.56\%                               & 87.03\%       & 47.38\%                               & 50.68\%                           \\
    Entropy FPR             & \textbf{39.36\%}      & $\pmb{\downarrow}$\textbf{25.61\%}    & 46.06\%       & 32.31\%                               & 33.09\%                           \\
    Entropy AUROC           & \textbf{78.22\%}      & $\pmb{\uparrow}$\textbf{85.41\%}      & 77.96\%       & $\pmb{\downarrow}$\textbf{60.84\%}    & $\pmb{\downarrow}$\textbf{61.40\%}\\
    \midrule
    Max. Score Pre          & 68.30\%               & 77.32\%                               & 68.44\%       & 63.96\%                               & 59.13\%                           \\
    Max. Score Rec          & 83.97\%               & 81.83\%                               & 83.87\%       & 65.55\%                               & 56.66\%                           \\
    Max. Score FPR          & \textbf{38.97\%}      & $\pmb{\downarrow}$\textbf{24.00\%}    & 38.68\%       & 36.93\%                               & 39.16\%                           \\
    Max. Score AUROC        & \textbf{78.12\%}      & $\pmb{\uparrow}$\textbf{85.63\%}      & 78.15\%       & $\pmb{\downarrow}$\textbf{69.80\%}    & $\pmb{\downarrow}$\textbf{61.84\%}\\
     \midrule
    Top-3 Scores Pre        & 67.48\%               & 76.36\%                               & 67.88\%       & 68.48\%                               & 59.41\%                           \\
    Top-3 Scores Rec        & 85.81\%               & 85.71\%                               & 84.60\%       & 85.08\%                               & 55.39\%                           \\
    Top-3 Scores FPR        & \textbf{41.35\%}      & $\pmb{\downarrow}$\textbf{26.53\%}    & 40.04\%       & 39.16\%                               & 37.85\%                           \\
    Top-3 Scores AUROC      & \textbf{78.29\%}      & $\pmb{\uparrow}$\textbf{86.24\%}      & 78.38\%       & 79.60\%                               & $\pmb{\downarrow}$\textbf{61.86\%}\\
    \bottomrule
    \end{tabular}
    }
    \caption{Training and attack metrics for ResNet-50 target models trained on Stanford Dogs. We compare the results for the standard model to models trained with label smoothing (LS) and Laplace approximation (LA) as calibration techniques and temperature scaling (Temp) and L2 regularization as defense techniques. Arrows indicate the differences compared to the standard model.}
    \label{tab:model_attack_metrics_dogs}
    \vskip -0.1in
\end{table}

{\bf (Q2) High Prediction Scores Lower Privacy Risks.} \label{sec:overconfidence_distracts_mi}
To shed light on the connection between overconfidence and high FPR of the MIAs, we analyzed the mean maximum prediction scores (MMPS) of the target models' predictions.

Tab.~\ref{tab:mmc_top_3} shows the MMPS values measured on a standard ResNet-50 and underlines our assumption that all score-based MIAs against models trained with standard procedure mainly rely on the maximum score since there is a clear difference between the MMPS of false-positive and true-negative predictions. The results we have obtained for the CIFAR-10 models are similar to the Results on the ResNet-50 and we present these in Appx.~\ref{app:cifar_mmc}. For additional results on the ResNet-50, see Appx.~\ref{app:dogs_mmc}. 

It seems that the non-maximum scores are not providing significant information on the membership status since the MMPS values of the false-positive predicted samples using the maximum score attack and the top-3 attack differ only slightly. Modifying the top-3 attack to use a larger part of the score vector for inferring membership of the samples did not significantly improve the membership inference either. 

So on one side, neural networks are overconfident in their predictions, even on inputs without any known content. It prevents a reasonable interpretation regarding a model's probability of being correct in its predictions. During MIAs, on the other side, this behavior implicitly protects the training data since the information content of the prediction score is rather low. Consequently, there is a trade-off between a model's ability to react to unknown inputs and its privacy leakage. We explore this trade-off in Q3. We further argue that any adversarial example maximizing the target model's scores in an arbitrary class would also be classified as a member in almost all cases. So it is possible to hide members in a larger dataset of non-members that are altered by adversarial attacks to maximize the target model's scores.

\begin{table}[t]
    \centering
    \resizebox{\columnwidth}{!}{
    \begin{tabular}{llcc}
    \toprule
    \textbf{Dataset}                & \textbf{Attack}   & \textbf{FP MMPS}  & \textbf{TN MMPS}      \\
    \midrule
    \multirow{3}{*}{Stanford Dogs}  & Entropy           & 0.9984                & 0.7565            \\
                                    & Max. Score        & 0.9985                & 0.7580            \\
                                    & Top-3 Scores      & 0.9979                & 0.7486            \\
    \midrule
    \multirow{3}{*}{Fake Dogs}      & Entropy           & 0.9977                & 0.7700            \\
                                    & Max. Score        & 0.9979                & 0.7724            \\
                                    & Top-3 Scores      & 0.9971                & 0.7648            \\
    \midrule 
    \multirow{3}{*}{AFHQ Cats}      & Entropy           & 0.9972                & 0.7205            \\
                                    & Max. Score        & 0.9972                & 0.7208            \\
                                    & Top-3 Scores      & 0.9959                & 0.7137            \\
    \bottomrule
    \end{tabular}
    }
    
  \caption{MMPS for false-positive (FP) and true-negative (TN) predictions of different attacks on the standard ResNet-50 model on selected datasets. A clear difference between false-positive and true-negative mean maximum prediction scores for all attacks can be seen. This indicates that all of the analyzed attacks heavily relied on the maximum prediction score.}
  \label{tab:mmc_top_3}
\end{table}

\begin{figure}[ht]
    \centering
    \begin{subfigure}[b]{.49\linewidth}
        \centering
        \includegraphics[width=\linewidth]{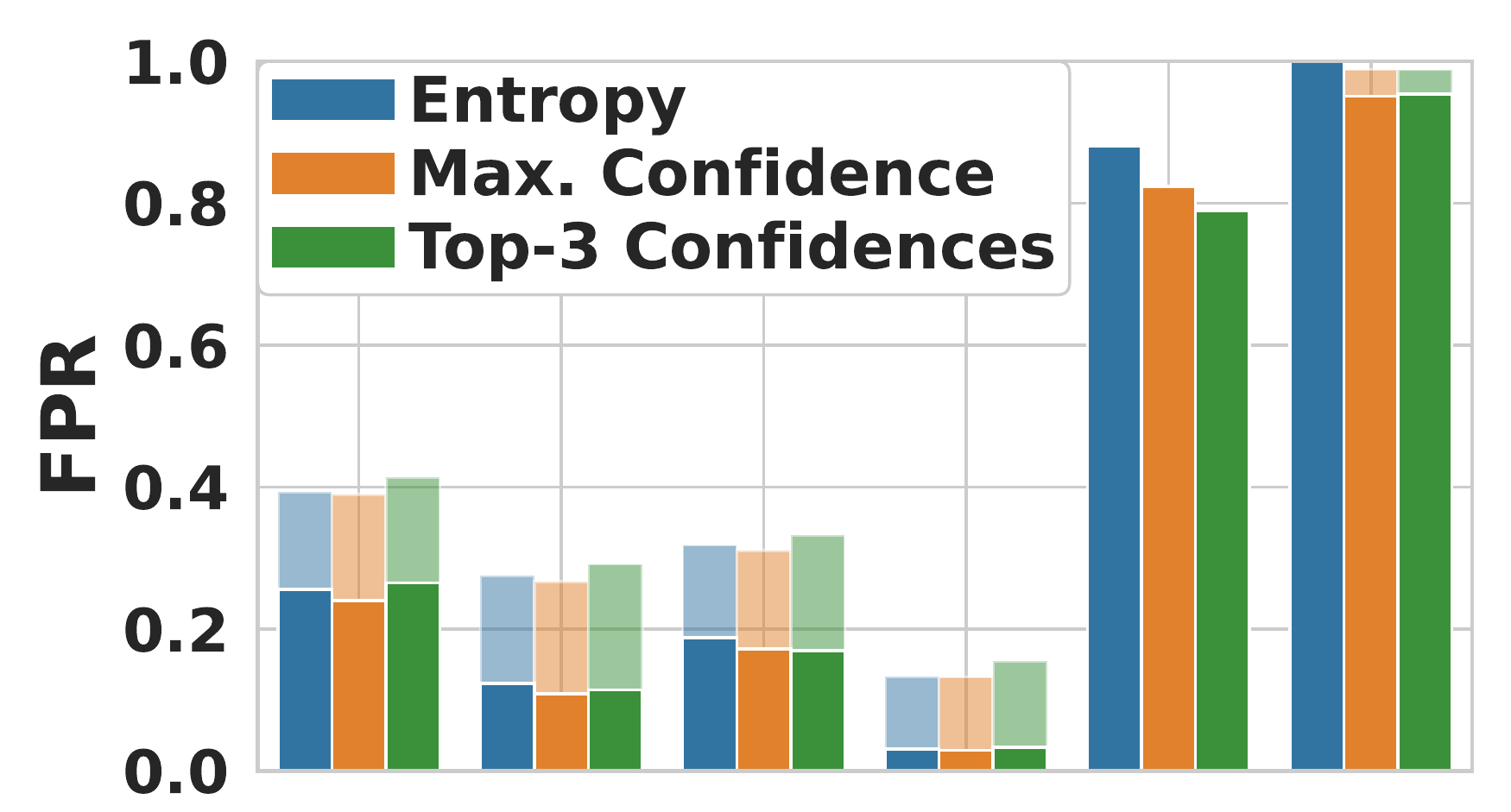}
        \captionsetup{justification=centering}
        \caption{ResNet-50 (LS)}
        \label{fig:fpr_resnet50_dogs_ls}
    \end{subfigure}
    \begin{subfigure}[b]{.475\linewidth}
        \hfill
        \includegraphics[width=.92\linewidth, height=21.4mm]{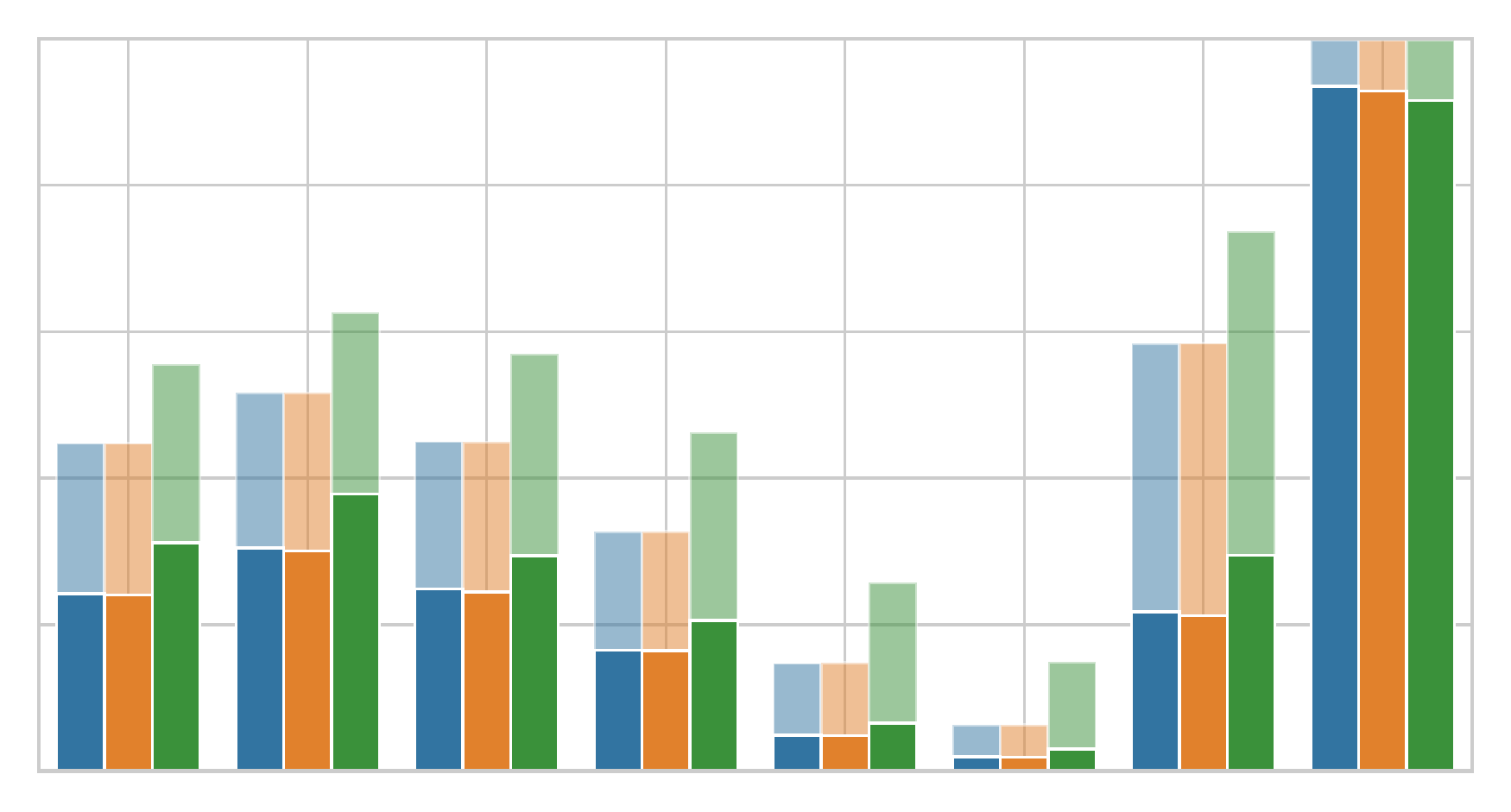}
        \captionsetup{justification=centering}
        \caption{ResNet-18 (LS)}
        \label{fig:fpr_resnet18_cifar10_ls}
    \end{subfigure}
    \begin{subfigure}[b]{.49\linewidth}
        \centering
        \includegraphics[width=\linewidth]{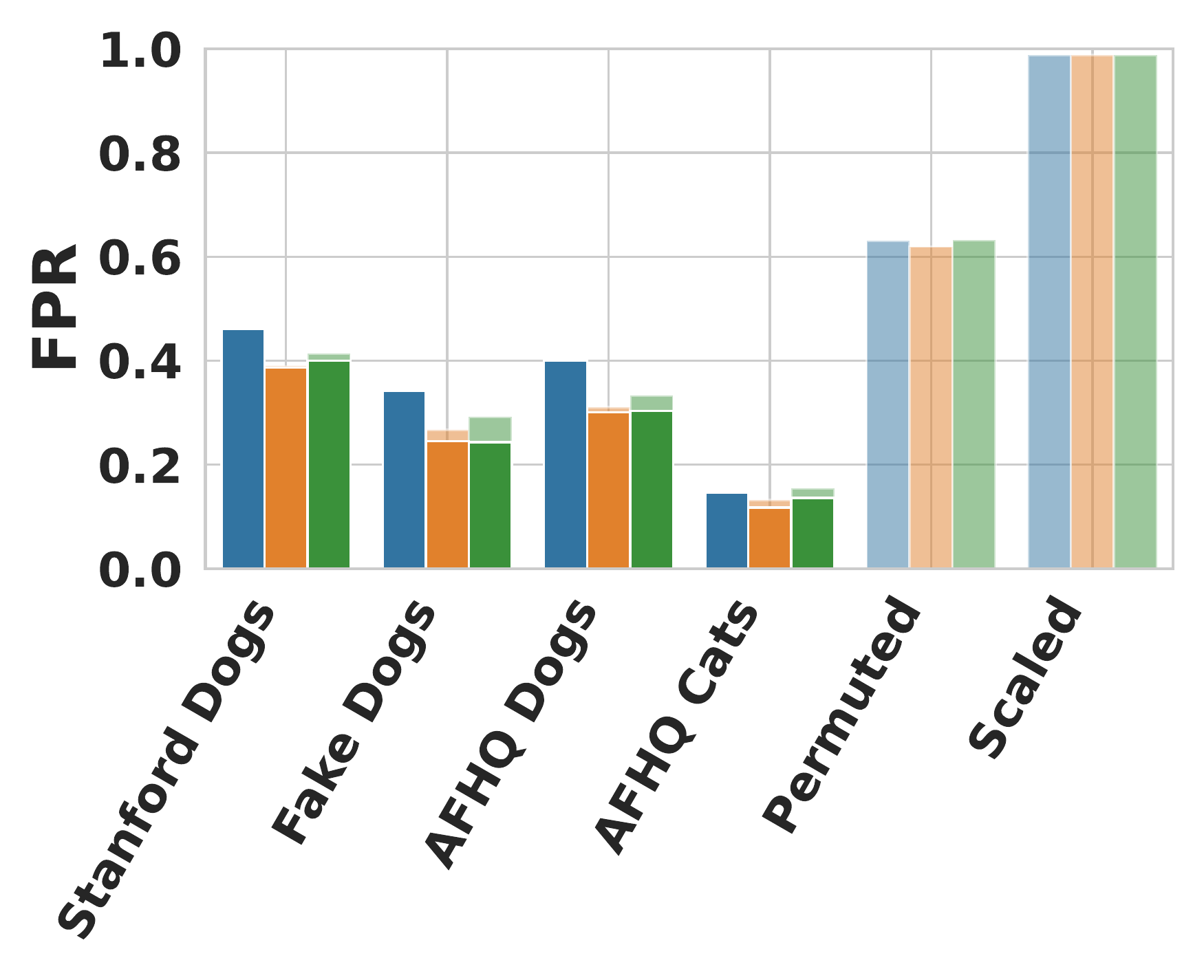}
        \captionsetup{justification=centering}
        \caption{ResNet-50 (LA)}
        \label{fig:fpr_resnet50_dogs_la}
    \end{subfigure}
    \begin{subfigure}[b]{.48\linewidth}
        \hfill
        \includegraphics[width=\linewidth]{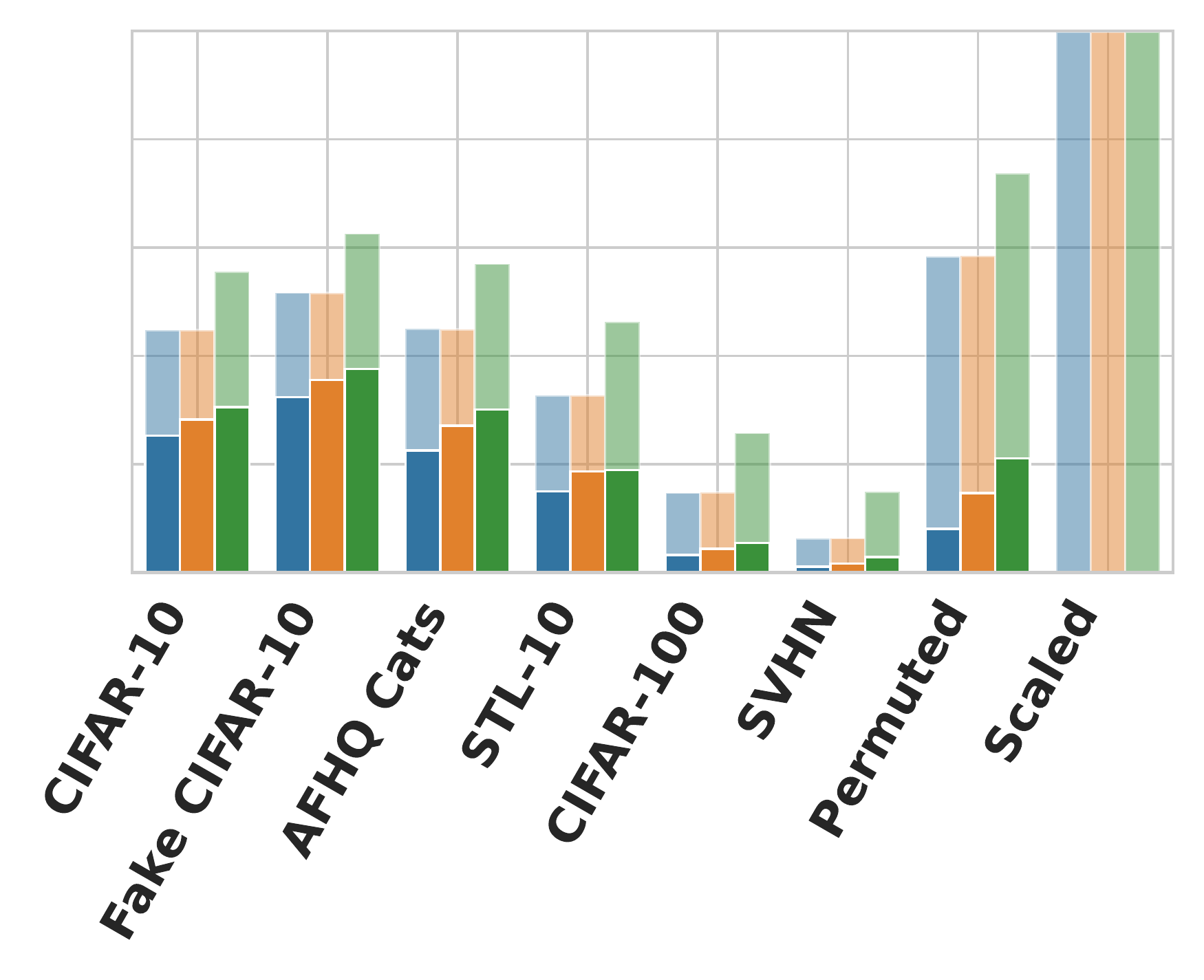}
        \captionsetup{justification=centering}
        \caption{ResNet-18 (LA)}
        \label{fig:fpr_resnet18_cifar10_llla}
    \end{subfigure}
    \caption{False-positive rates (FPR) of MIAs against ResNet-18 and ResNet-50. The transparent bars represent the FPR of the standard models, whereas the solid bars represent the FPR of the models with the respective modification given in parentheses - label smoothing (LS) and Laplace approximation (LA). Both calibration methods reduce the FPR for almost all inputs.}
    \vskip -0.1in
\end{figure}
\begin{figure*}[t]
     \centering
     \begin{subfigure}[b]{0.247\textwidth}
         \centering
         \includegraphics[width=\textwidth]{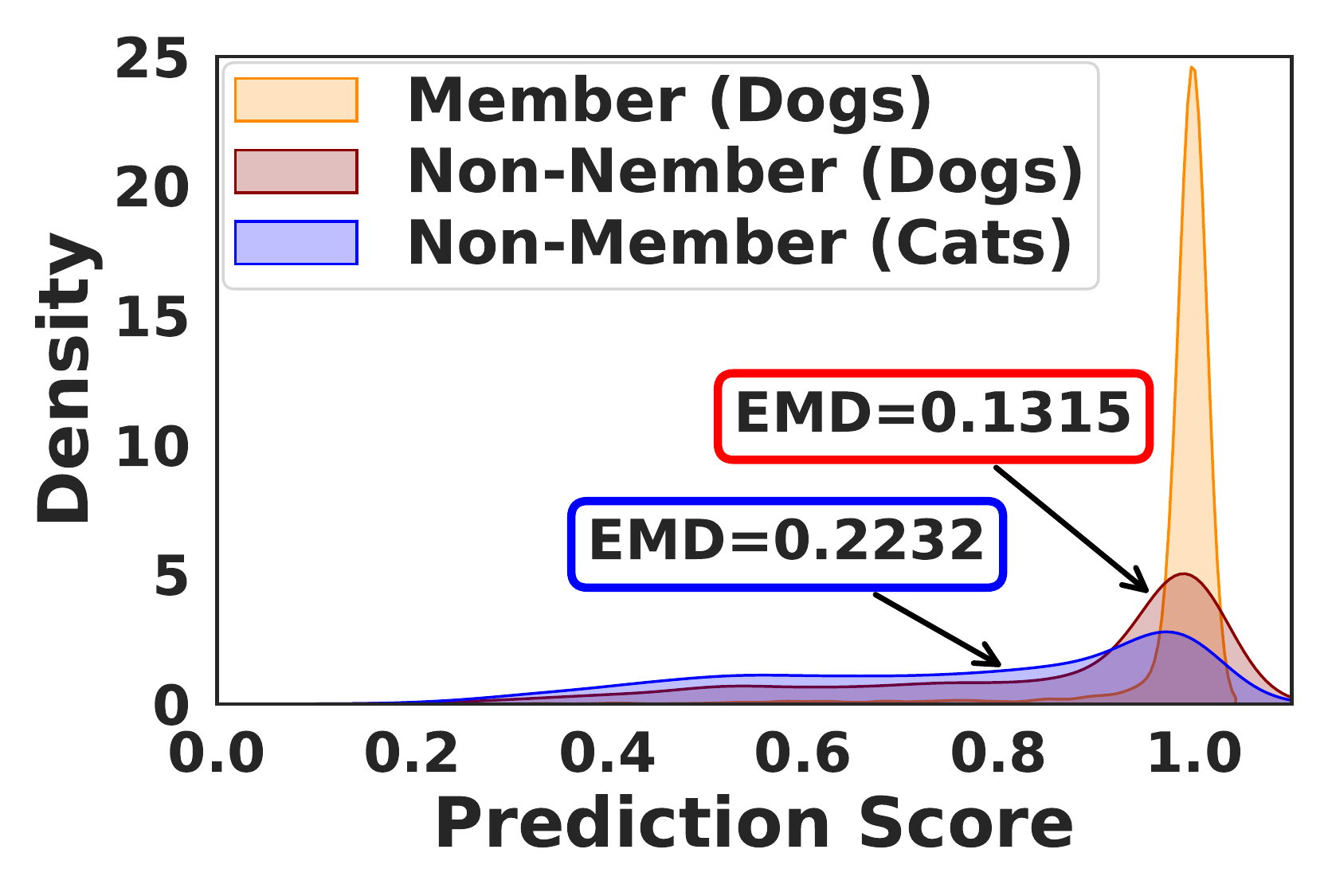}
         \captionsetup{justification=centering}
         \caption{ResNet-50}
         \label{fig:kde_standard}
     \end{subfigure}
     \begin{subfigure}[b]{0.24\textwidth}
         \centering
         \includegraphics[width=\textwidth]{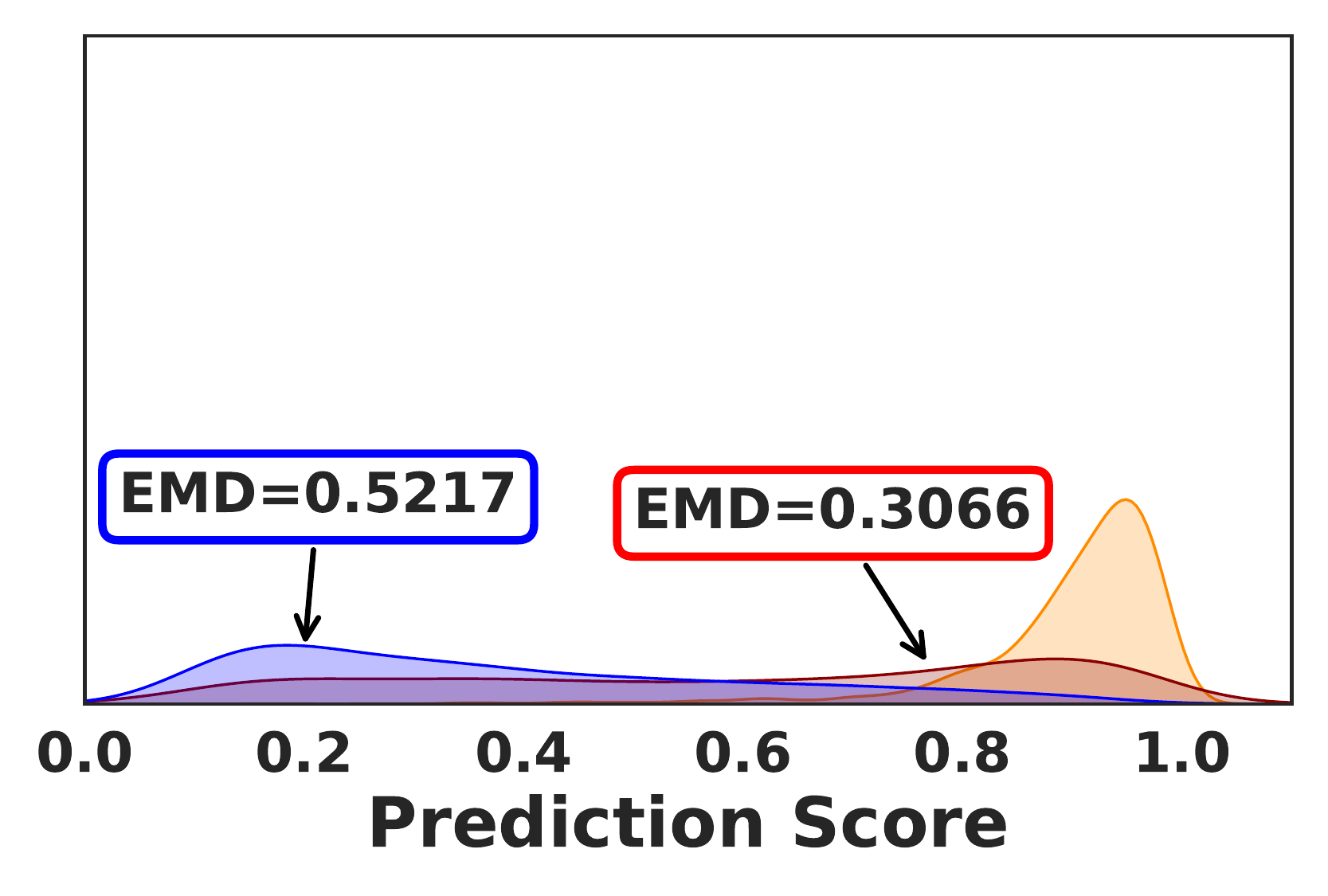}
         \captionsetup{justification=centering}
         \caption{ResNet-50 (LS)}
         \label{fig:kde_label_smoothing}
     \end{subfigure}
    \begin{subfigure}[b]{0.24\textwidth}
         \centering
         \includegraphics[width=\textwidth]{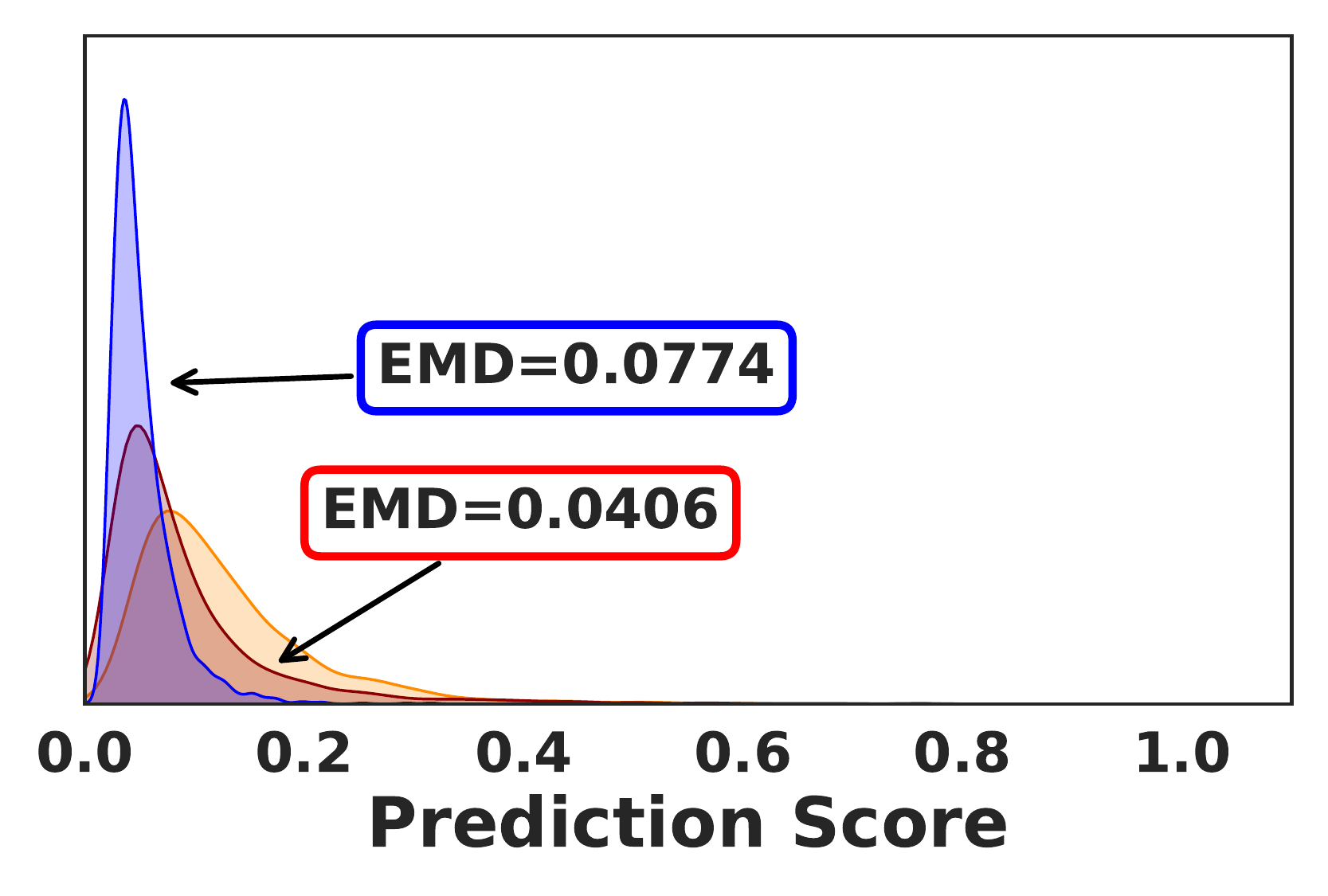}
         \captionsetup{justification=centering}
         \caption{ResNet-50 (Temperature)}
         \label{fig:kde_temperature}
     \end{subfigure}
    \begin{subfigure}[b]{0.24\textwidth}
         \centering
         \includegraphics[width=\textwidth]{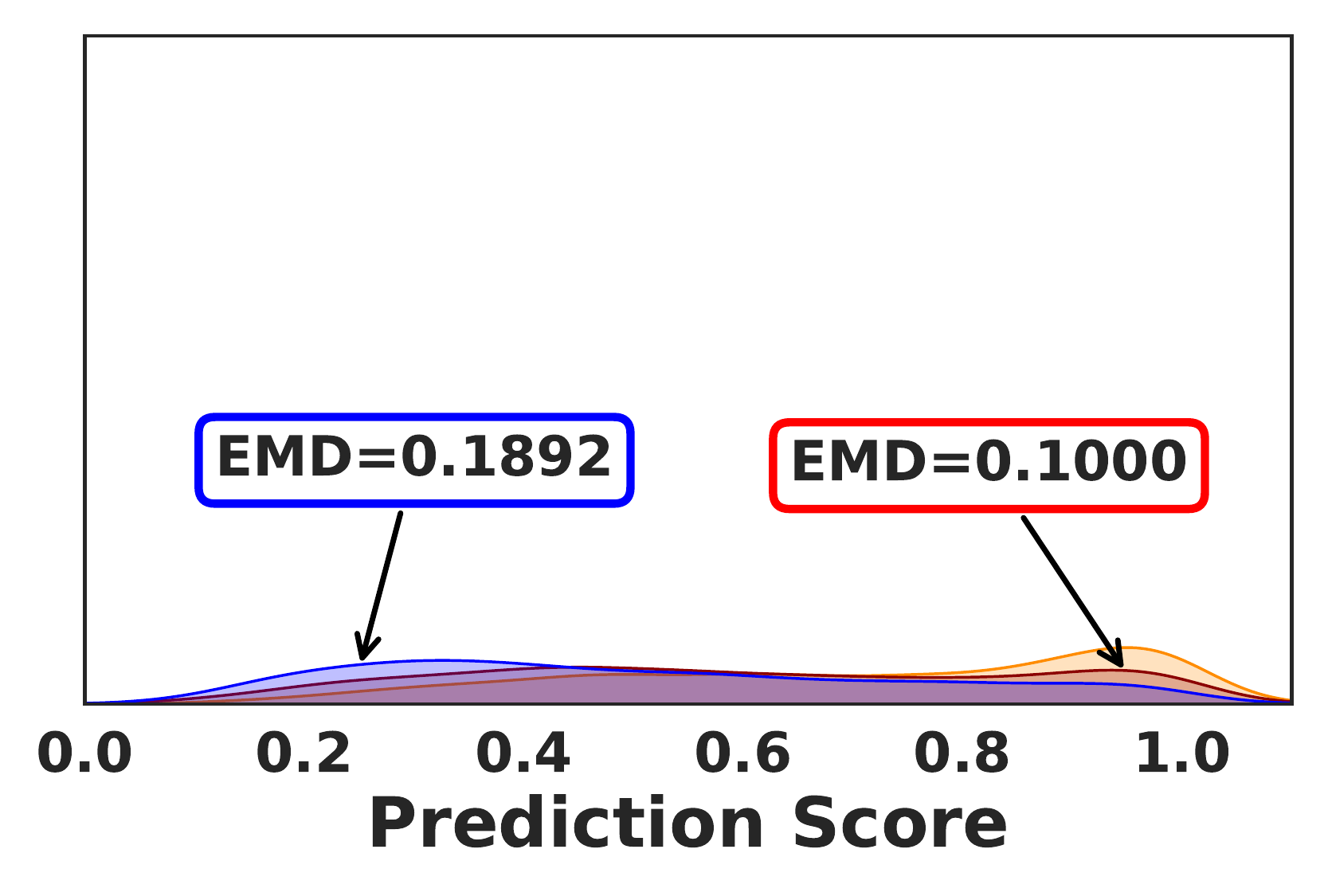}
         \captionsetup{justification=centering}
         \caption{ResNet-50 (L2)}
         \label{fig:kde_l2_regularization}
     \end{subfigure}

    \caption{Kernel density estimation applying Gaussian kernels on the top prediction scores values of ResNet-50 target models. We use equally-sized member and non-member subsets of Stanford Dogs and AFHQ Cats. We further state the earth mover's distance (EMD) between each dataset and the member dataset. Label smoothing (LS) moves the non-member distributions further away, and consequently, the members become easier to separate. Temperature scaling and L2 regularization show an inverse effect and increase the overlapping.}
    \label{fig:kde}
\end{figure*}
\begin{figure}[ht]
     \centering
     \begin{subfigure}[b]{.32\linewidth}
         \centering
         \includegraphics[width=\textwidth]{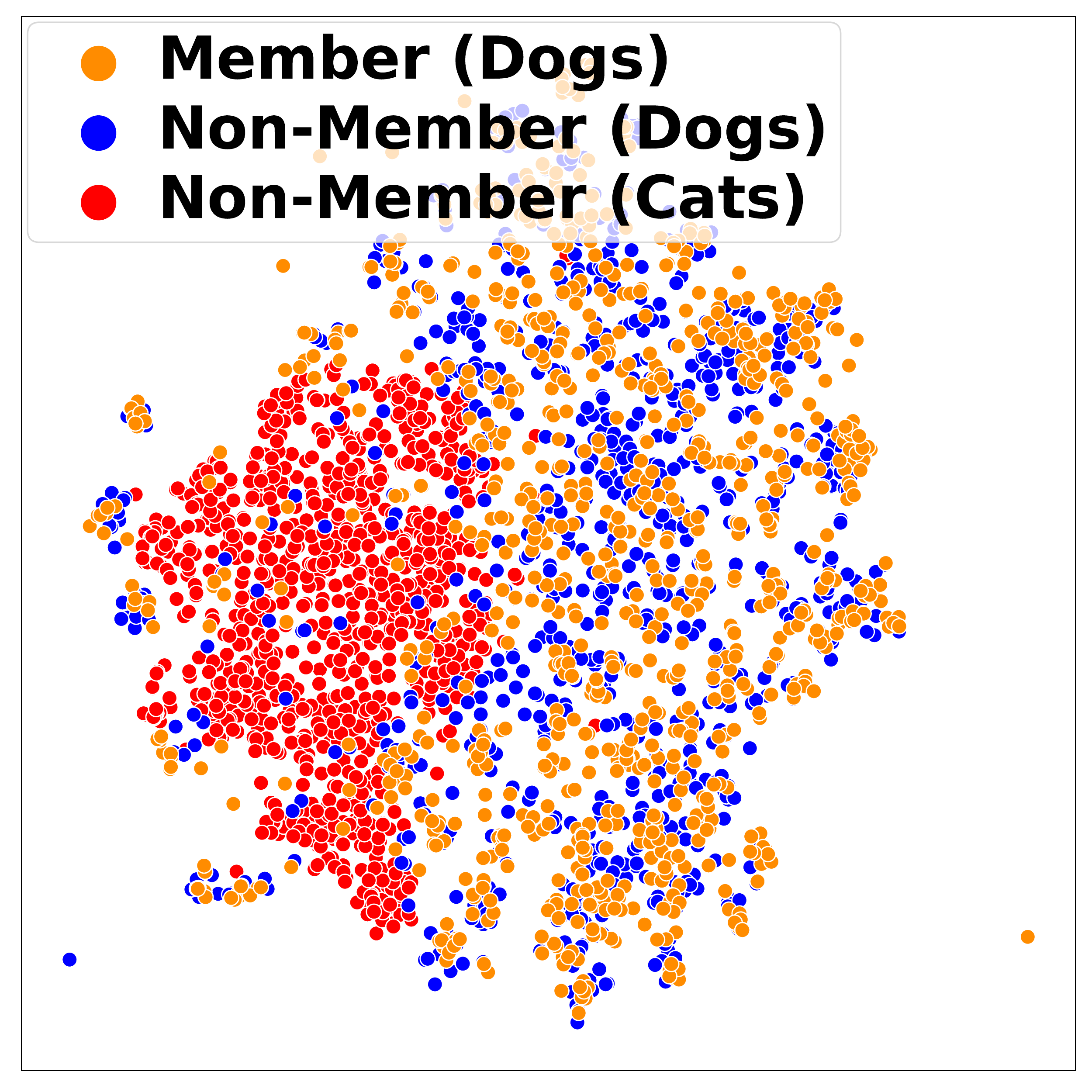}
         \caption{ResNet-50}
         \label{fig:tsne_standard}
     \end{subfigure}
     \begin{subfigure}[b]{.32\linewidth}
         \centering
         \includegraphics[width=\textwidth]{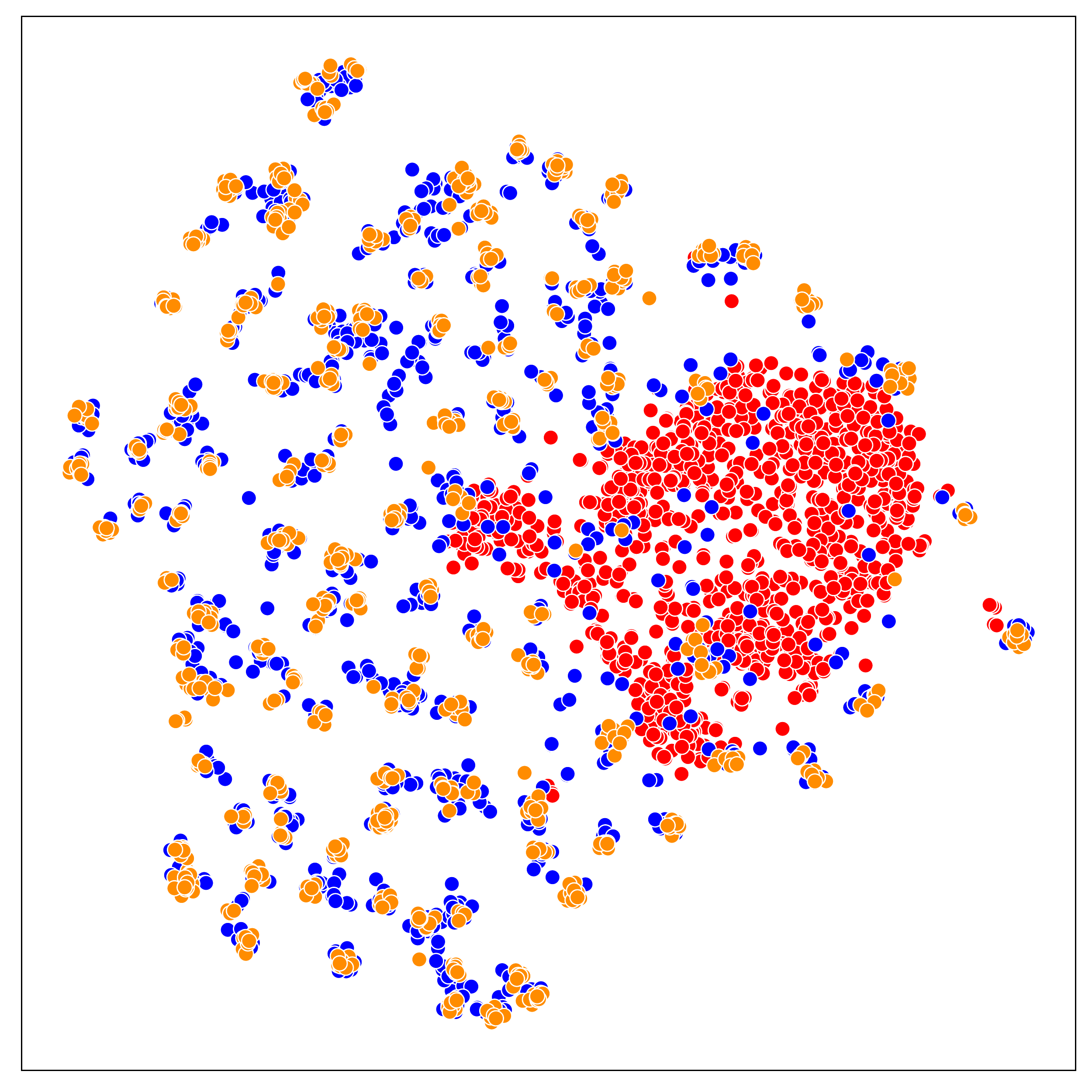}
         \captionsetup{justification=centering}
         \caption{ResNet-50 (LS)}
         \label{fig:tsne_label_smoothing}
     \end{subfigure}
     \begin{subfigure}[b]{.32\linewidth}
         \centering
         \includegraphics[width=\textwidth]{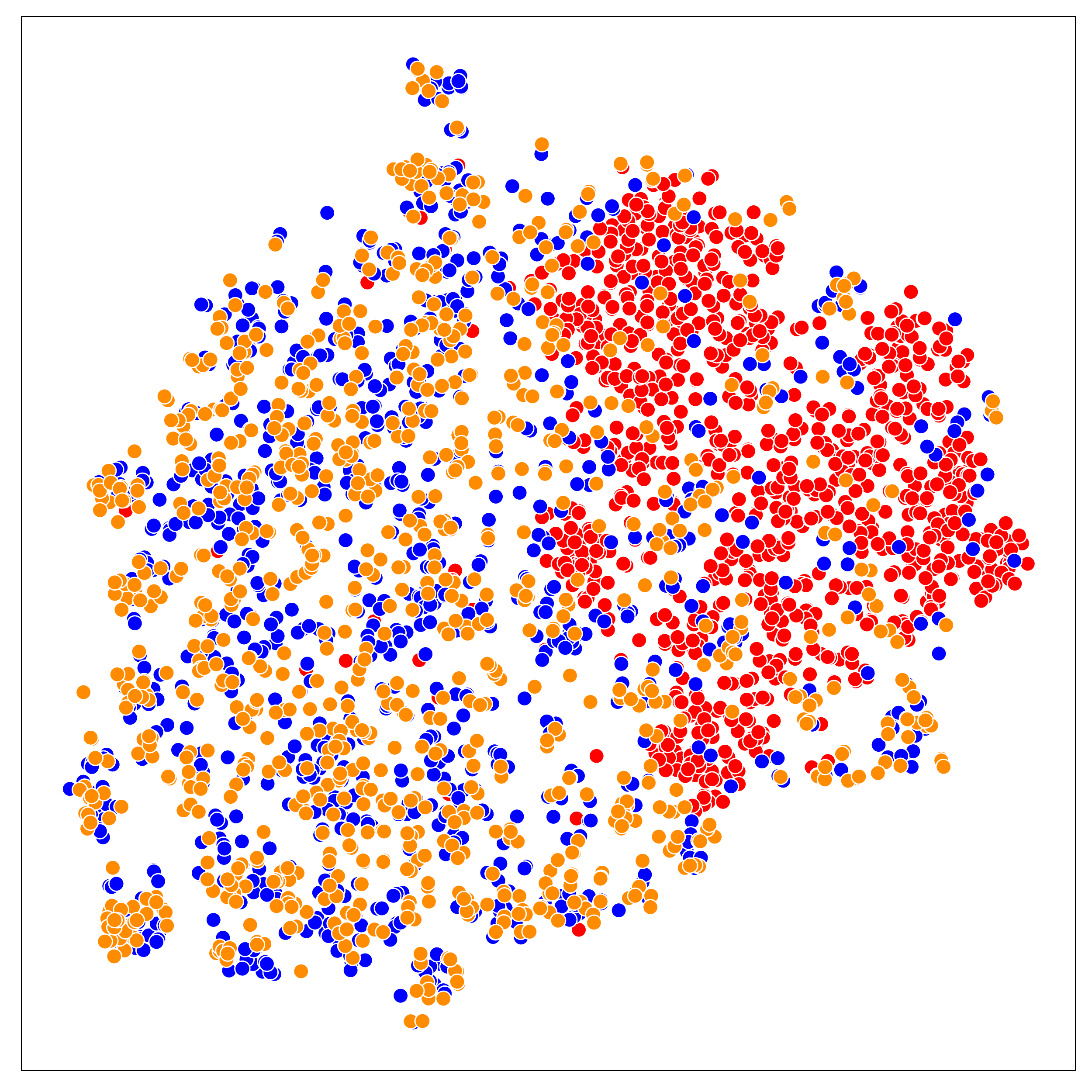}
        \captionsetup{justification=centering}
         \caption{ResNet-50 (L2)}
         \label{fig:tsne_regularization}
     \end{subfigure}
    \caption{T-SNE visualization of the penultimate ResNet-50 layer activations on training samples (orange), test samples (blue), and OOD samples (red). Label smoothing (LS) creates much tighter clusters of training and OOD cat samples, which makes them easier to separate, whereas L2 regularization has an reverse effect.}
    \label{fig:tsne}
\end{figure}

\newpage
{\bf (Q3) Mitigating Overconfidence Increases Privacy Risks.} \label{sec:calibration_effect}
Ideally, neural networks are properly calibrated, and their prediction scores represent the probabilities of correct predictions. To calibrate the models and to reduce the overconfidence, we retrained the ResNet-18 and ResNet-50 models with label smoothing. We performed the same calibration method on both the target and the shadow models, which reflects a worst-case scenario, with an adversary knowing the exact calibration method and hyperparameters.  

Label smoothing not only calibrates a model but may also improve its test accuracy, as shown in Tab.~\ref{tab:model_attack_metrics_dogs} for ResNet-50. Detailed results for ResNet-18 are given in Appx.~\ref{app:resnet_18_add_attack_result}. 

Both the expected calibration error (ECE) and overconfidence error (OE) dropped significantly, demonstrating a strong calibration effect when using label smoothing. 

Previous works on MIAs suggested that minimizing the accuracy gap between the training and test accuracy on the same architecture leads to weaker attacks and, therefore, to lower privacy risks. 
However, as demonstrated by the results summarized in Tab. \ref{tab:model_attack_metrics_dogs}, label smoothing improves the test accuracy and still yields higher attack precision values for all three attacks on both architectures. Figs.~\ref{fig:fpr_resnet50_dogs_ls} and \ref{fig:fpr_resnet18_cifar10_ls} further illustrate that label smoothing reduces the number of false-positive membership predictions. 
Whereas the FPR on the Permuted samples is drastically reduced for ResNet-18, the FPR of the ResNet-50 on the Permuted samples even increases when using label smoothing. We note that this effect does only occur in some training runs. In other cases, the FPR for Permuted data drops similar to the ResNet-18 results. On all datasets, the reductions in the FPR are comparable between the ResNet-18 and ResNet-50. The FPR also decreases for inputs similar to the training data. For comparison, we also apply a Laplace approximation (LA) on the weights of the final layers to mitigate overconfidence. As shown in Figs.~\ref{fig:fpr_resnet50_dogs_la} and \ref{fig:fpr_resnet18_cifar10_llla}, LA is better suited to avoid high prediction scores on the Permuted and Scaled samples. 

Our results demonstrate that if a model shows reduced prediction scores on unseen inputs, the samples of the training data are easier to identify. It reduces the protection induced by overconfident predictions (on unseen inputs) and increases vulnerability to MIAs. We applied a kernel density estimation (KDE) to visualize the distribution of the maximum prediction scores of the ResNet-50 target models on member and non-member data. Figs.~\ref{fig:kde_standard} and \ref{fig:kde_label_smoothing} show the estimated density functions. Without label smoothing, all three distributions have their mode around prediction scores of $1.0$. This leads to a large overlap of the distributions. Samples with prediction scores this high are most likely classified as false-positive members as the FP MMPS values in Tab.~\ref{tab:mmc_top_3} suggest. We also state the earth mover's distance (EMD) in the KDE plots to quantify the distance between the member and non-member distributions. Label smoothing separates the three distributions clearly and doubles both EMD values. The label smoothing model tends to be less overconfident in its predictions on unknown input data, and hence the member samples are easier to separate from non-members. This increases the potential privacy leakage of MIAs. 

As depicted in Fig.~\ref{fig:tsne}, we further used t-SNE \citep{vandermaaten08a} to plot the penultimate layer activations on samples from the same datasets as used for the KDE plots. Whereas the standard model in Fig.~\ref{fig:tsne_standard} shows an overlapping between the activations of the three datasets, label smoothing in Fig.~\ref{fig:tsne_label_smoothing} creates tighter clusters of dog samples and separates the OOD cat images more clearly.

{\bf (Q4) A Trade-off Between Calibration and Defenses Exists.} \label{sec:defenses}
Whereas calibration tries to maximize the informative value of the prediction scores, many defenses against MIAs aim to reduce the informative value and to align the score distributions of members and non-members. In this section, we want to investigate whether it is possible to defend calibrated models or a trade-off between calibration and defenses against MIAs exists. Defenses reduce the generalization of a model in terms of its ability to distinguish between samples from known and unknown inputs and express meaningful scores. To test this, we first applied temperature scaling with $T=10$ to the trained ResNet-50 standard model without calibration. Fig.~\ref{fig:kde_temperature} shows the estimated maximum prediction score distributions. The score vectors converge to a uniform distribution, and the distributions of the top scores are much more similar. This can be seen by the significantly lower EMD values. With an ECE of 51\% using temperature scaling, the information content of the actual prediction score is greatly reduced, and the AUROC for the Entropy and Maximum Score attacks drop significantly, as shown in Tab.~\ref{tab:model_attack_metrics_dogs}. On the top-3 score attack, temperature scaling has no effect. We suspect this is due to the added temperature term being a monotone transformation, not removing information encoded in the top-3 score patterns.

We also investigated L2 regularization as a stronger defense applied during training on our ResNet models. L2 regularization effectively reduces the vulnerability to MIAs. For all attacks, both precision and recall drop significantly at the cost of reduced test accuracy, as Tab.~\ref{tab:model_attack_metrics_dogs} states. Moreover, the ECE and OE are significantly higher than for the model trained with label smoothing. The distribution of the highest prediction scores can be seen in Fig.~\ref{fig:kde_l2_regularization}. Similar to temperature scaling, L2 regularization aligns the distributions of members and non-members but distributes the maximum scores more equally instead of pushing it towards a single value. Fig.~\ref{fig:tsne_regularization} shows a similar effect of overlapping distributions in the penultimate layer activations, making it harder to separate members from non-members and OOD data.

As shown in our experiments, defenses are contrary to calibration. Our results indicate that a trade-off exists between defending models against MIAs and applying calibration to increase the model's informative value.

\section{Discussion}\label{sec:discussion}
In all our analyses, we followed the standard threat model for MIAs in the literature and assumed a strong adversary with full knowledge about the target model's architecture and training procedure and having access to data from the target's training distribution. 
Our experiments underline the known fact that modern neural networks are not inherently able to identify unseen and unknown inputs and cannot adapt their behavior in terms of reducing the prediction scores. However, we have shown that this is why the expressiveness of MIAs in realistic scenarios is greatly reduced, and the associated privacy risks are thus much lower than previously assumed. Loosening the attack scenario assumptions and providing the attacker with even less information during an attack, the effectiveness of MIAs will decrease even further.

One way to mitigate the problem of false-positive predictions on unseen data is to first try to identify and remove all OOD samples. This would indeed prevent some false-positive predictions caused by completely different data distributions. However, we demonstrated that the problem of high FPR also occurs on datasets similar to the training data. In this case, the adversary has no means to tell whether a given sample is in- or out-of-distribution if the images' contents are similar, which in turn makes it impossible for the attacker to filter out OOD samples. Even if this were possible, by generating synthetic images, we have shown that there is a potentially unlimited number of samples that follow the training distribution and still lead to false-positive MIA predictions, questioning the overall informative value of MIAs. 

We only considered prediction score-based MIAs, but we expect our results to be similar for other kinds of attacks. Doing so provides an interesting avenue for future work. Also, future research should further investigate the trade-off between MIA defenses and calibration of machine learning models and how both aspects could be balanced. Furthermore, including techniques from open set recognition and OOD detection into MIAs might improve their effectiveness.

\section{Conclusion}\label{sec:conclusion}
We have shown that MIAs produce high false-positive rates due to overconfident predictions of modern neural networks for in- and out-of-distribution data. In stark contrast to previous works stating strong attack results on standard neural networks, we demonstrate that MIAs are actually not reliable in realistic scenarios, and overconfidence can be seen as a natural defense against these attacks. 
Our results suggest that there is a trade-off between reducing a model's overconfidence and its susceptibility to MIAs. Therefore, the informative value of MIAs increases on calibrated models, increasing the privacy risk. 
As a result, our analysis has shown that MIAs are not as powerful as previously thought and are at odds with the meaning of neural networks' prediction scores.

\newpage

\section* {Acknowledgements} 
This work was supported by the German Ministry of Education and Research (BMBF) within the framework program ``Research for Civil Security'' of the German Federal Government, project KISTRA (reference no. 13N15343).

\bibliographystyle{named}
\bibliography{main.bib}

\clearpage
\appendix


\section{Extended Proof of Theorem 1}\label{app:proof}
We provide an extended proof of our Theorem 1, particularly for the following part:  \\

\textit{
For almost any input $x \in \mathbb{R}^m$ and a sufficiently small $\epsilon>0$ if $\max_{i=1,...,d} \,\, f(x)_i \geq 1-\epsilon$, it follows that $h(f(x)) = 1$.
}\\

Most score-based membership inference attacks (MIAs) ${h: \mathbb{R}^d \rightarrow \{0, 1\}}$ are solely based on thresholds on prediction score statistics. In our work, we investigated attacks using thresholds on the maximum value and the entropy of the prediction score vectors. For the maximum prediction score attack, the statement is trivial since for a given attack threshold $\tau$ on the maximum score, the attack's decision is
\begin{equation}
h(f(x))=
\begin{cases}
1,& \text{if} \max_{i=1,...,d} \,\, f(x)_i \geq \tau \\
0, & \text{otherwise}\\
\end{cases}
\end{equation}
Given any input $x\in \mathbb{R}^m$ and $\max_{i=1,...,d} \,\, f(x)_i \geq 1-\epsilon$ with $\epsilon>0$, it follows that $h(f(x))=1 \iff \epsilon \leq 1 - \tau$. The same can be shown analogously for other threshold-based attacks, such as the entropy-based attack. However, for attacks using a separate neural network as an attack model to predict the membership status based on the confidence vector, a closed-form analysis is non-trivial, and there might exist model inputs and corresponding prediction score vectors for which the attack model outputs a non-member prediction even if $\max_{i=1,...,d} \,\, f(x)_i \geq 1-\epsilon$ is fulfilled. For example, there might exist an adversarial perturbation of the confidence vector that misleads the attack model's prediction. Still, our results in Table 3 empirically support that the theorem also holds for attacks using a separate attack model.

The second part of our proof uses that ${lim_{\delta \rightarrow \infty} \max_{i=1,...,d} f(\delta x)_i = 1}$. We refer interested readers to the proofs of Lemma 3.1 and Theorem 3.1 in \citet{Hein2019WhyRN} and only outline the proof's intuition. As stated in our paper, a neural network using ReLU activation functions is decomposing the input space into a finite set of polytopes (linear regions). As a consequence of a finite number of polytopes representing an infinite input space, the outer polytopes extend to infinity. Scaling the input samples will move them into the outer polytopes, allowing arbitrarily high prediction scores.

\section{Experimental Setup Details}\label{app:experimental_details}
We performed all our experiments on NVIDIA DGX machines running NVIDIA DGX Server Version 4.4.0 and Ubuntu 18.04 LTS. The machines have 1.6TB of RAM and contain Tesla V100-SXM3-32GB-H GPUs and Intel Xeon Platinum 8174 CPUs. We further relied on Python 3.8.8 and PyTorch 1.8.1 with Torchvision 0.9.1  \citep{paszke2019pytorch} for the implementation and training of the neural networks. We provide a dockerfile together with our code to facilitate execution and reproducibility. We performed a single experimental run and set the seed for all experiments to 42 to allow reproducibility.

\subsection{Architectures}\label{app:nn_architectures}
We use ResNet-50, ResNet-18, EfficientNetB0 and a custom CNN (SalemCNN) for our experiments.

\textbf{ResNet-50}: We use the ResNet-50 implementation and the ImageNet weights provided by PyTorch. 

\textbf{ResNet-18 and EfficientNetB0}: We rely on the implementations provided at \url{https://github.com/kuangliu/pytorch-cifar} under MIT License. Note that this ResNet-18 and EfficientNetB0 implementations slightly differ from the official architectures since we train on CIFAR-10 instead of ImageNet. Differences occur mainly in early layers and show up in smaller kernel sizes and strides to avoid a large reduction in feature map sizes.

\textbf{SalemCNN}: Following \citet{salem2018mlleaks}, the model consists of two convolutional layers, each containing 32 filters of size 5 and a padding of 2 to maintain the spatial ratio. Each layer is followed by a ReLU activation and a $2\times 2$ max pooling layer with stride 2 for downsampling. After that, two fully-connected layers further process the extracted features. The first fully-connected layer contains 128 neurons, while the number of neurons of the second one corresponds to the number of classes on the training set. Note that in its original version, the model uses a tanh activation on the first fully-connected layer. We change it to a ReLU to keep the network piecewise linear. We did not notice any significant performance differences between both variants.

\subsection{Attacks}\label{app:salem_details}
{\bf Top-3 Score Attack} We use a simple neural network as an inference model for membership inference. The model consists of a neural network with one hidden layer containing 64 neurons and ReLU activations. Unlike \citet{salem2018mlleaks}, our inference model only uses a single output neuron, followed by a sigmoid function. During training, we first query the shadow model with samples with known membership status and collect the predicted scores. The values of each score vector are then sorted in descending order and the three highest values are used together with the membership status to train the inference model. The inference model is then trained on the membership dataset gathered by querying the shadow model and collecting the prediction score vectors. We use Adam optimizer \citep{kingma2017adam} with learning rate 0.01, optimizing a binary cross-entropy loss. The training uses a batch size of 16 and is stopped if the loss is not decreasing by at least $5e^{-4}$ for 15 epochs.

{\bf Maximum Prediction Score Attack} To find the threshold for the maximum prediction score attack a receiver operating characteristic (ROC) curve is created with the maximum values of each prediction score vector. We then choose the best threshold that maximizes the true-positive rate while minimizing the false-positive rate.

{\bf Entropy Attack} For each of the prediction score vectors the entropy is calculated. To find the threshold a linear search is performed.

\subsection{Training Hyperparemters}\label{app:hyperparameters}
To set the hyperparameters, we perform a small grid search for each model and dataset. Hence, we do not aim to achieve maximum test accuracy but to keep the training procedure simple. We roughly optimized the number of epochs \{50, 100, 200\}, learning rate \{0.1, 0.05, 0.01, 0.005, 0.001, 0.0005, 0.0001\} and optimizer type \{Adam, SGD\}. Since we usually expect the MIAs to perform worse on models with higher test accuracy, we argue that further hyperparameter optimization would degrade the attack metrics and strengthen our statements. We finally choose Adam optimizer \citep{kingma2017adam} with default parameters ($\beta_1=0.9, \beta_2=0.999, \epsilon=1e-08$) to train all our models. We use a batch size of 64 and a seed of 42 for training and experiments. All models are trained for 100 epochs.

For training with label smoothing, we set the smoothing factors to $\alpha = 0.1$ (ResNet-50) and $\alpha = \frac{0.1}{12}$ (ResNet-18). For L2 regularization, we retrained the ResNets with a weight decay of $\lambda = 0.001$ for ResNet-50 and a weight decay of $0.0003$ for ResNet-18.

\begin{figure*}[ht]
\centering
\includegraphics[width=0.85\linewidth]{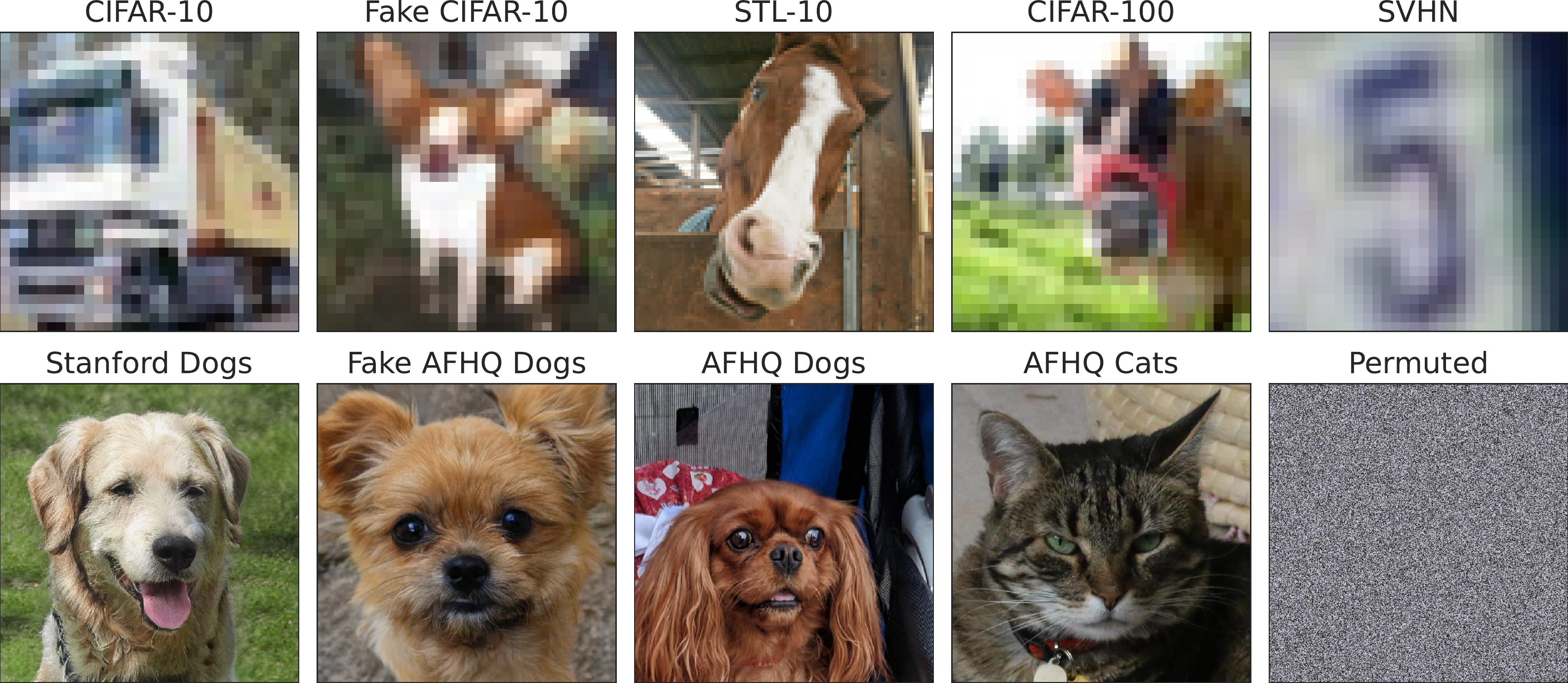}
\caption{Randomly selected samples from the datasets without any preprocessing. At inference time,  we scale all images to 32x32 for models trained on CIFAR-10 and 224x224 for models trained on Stanford Dogs. We further use StyleGAN2 to generate fake samples of CIFAR-10 and AFHQ Dogs, demonstrating a potentially infinite number of samples following a similar distribution.}
\label{fig:dataset_samples}
\end{figure*}

\textbf{CIFAR-10 models}: For the SalemCNN, we set the learning rate to 0.001, while using a higher learning rate of 0.01 for ResNet-18 and EfficientNetB0. 

\textbf{Stanford Dogs models}: We use pre-trained ImageNet weights for ResNet-50 and set the learning rate to 0.001. We replace the final fully-connected layer to match the number of classes. 

\subsection{KDE Plots}
We compute the earth mover's distance using the Wasserstein metric as implemented in SciPy 1.6.3 \citep{2020SciPy-NMeth} with default parameters. Kernel density estimations are created with Seaborn 0.11.1 \citep{Waskom2021} and default parameters on 2,058 samples for each dataset.

\subsection{t-SNE Plots}
We create all three t-SNE plots using scikit-learn 0.24.2 \citep{scikit-learn} with the same hyperparameters. Each embedding is initialized with a PCA. We set the perplexity to 30, the learning rate to 100, and the maximum number of iterations to 1,000. The random state is set to 13.

\subsection{Datasets}\label{app:datasets}

We normalize all input images on the statistics of the target model's training data (CIFAR-10 models) by computing the standard score $z=\frac{x-\mu}{\sigma}$ of each input or by using the imagenet statistics (Stanford Dogs models). We state exact parameters in table \ref{tab:statistics}. For inference on CIFAR-trained models, we downsize samples from other distributions to 32x32 pixels. For models trained on Stanford Dogs, we resize all inputs to 224x224 pixels.

\begin{table}[H]
\centering
\resizebox{\columnwidth}{!}{
\begin{tabular}{ l c c }
 \textbf{Dataset} & \textbf{Mean} & \textbf{Std} \\ 
  \hline
 CIFAR-10 & (0.4914, 0.4822, 0.4465) & (0.2470, 0.2435, 0.2616) \\  
 Stanford Dogs & (0.485, 0.456, 0.406) & (0.229, 0.224, 0.225) \\
\end{tabular}
}
\caption{Statistics for dataset normalization.}
\label{tab:statistics}
\end{table}

\textbf{CIFAR-10/ CIFAR-100} \citep{Krizhevsky09learningmultiple}: The CIFAR-10 and CIFAR-100 datasets  each consist of 60,000 color images of size 32x32. Both training and test splits contain 50,000 and 10,000 samples, respectively. CIFAR-10 samples are grouped into 10 classes, CIFAR-100 correspondingly into 100 classes. The number of samples per class is completely balanced. CIFAR-10 contains samples from the classes airplane, automobile, bird, cat, deer, dog, frog, horse, ship and truck. 

More information is available at \url{https://www.cs.toronto.edu/~kriz/cifar.html}.

\textbf{Stanford Dogs Dataset} \citep{KhoslaYaoJayadevaprakashFeiFei_FGVC2011}: The Stanford Dogs dataset contains 20,580 images of 120 different dog breeds. Number of samples per class is not balanced. The dataset is built on ImageNet samples and has a significantly higher image resolution than the other datasets we use, except AFHQ. We do not rely on the official dataset split to increase the number of training samples for the target and shadow models. We use 80\% of the data as training set, resulting in 16,464 samples for training and 4,116 samples as test data. We evaluate both, target and shadow models, on the full test split to keep it as large as possible. To improve generalization, we apply random rotation in a range of 20 degrees and resize the images so that the smaller side has a length of 230 pixels. We then randomly crop out a square image with size 224 pixels and flip it horizontally with 50\% probability. Besides training, we only resize the inputs to 224 pixels on the shorter side and center crop to obtain a square image with size 224 pixels. 

The dataset and a list of its classes are available at \url{http://vision.stanford.edu/aditya86/ImageNetDogs/}.

\textbf{Animal Faces-HQ (AFHQ)} \citep{choi2020starganv2}: The AFHQ dataset contains 16,130 images of animal faces of size 512x512 and split into 14,630 training and 1,500 test samples. We use the training split for our experiments, providing 4,739 dog, 5,153 cat and 4,738 wild animal samples.

\textbf{STL-10} \citep{pmlr-v15-coates11a}: The STL-10 dataset is inspired by CIFAR-10 and contains 96x96 color images of ten different classes. The classes are identical to CIFAR-10, except class \textit{monkey}. We therefore remove all samples containing monkeys. The full dataset contains a total of 5,000 labeled training samples, 8,000 test images and 100,000 unlabeled images from similar distributions for unsupervised learning. We use the training set for out experiments.

\textbf{SVHN} \citep{svhn}: The Street View House Numbers (SVHN) dataset provides over 600,000 digit images of cropped house numbers in natural scene images. The dataset consists of 73,257 training and 26,032 test images. We use the training set for our experiments.

\textbf{Fake CIFAR-10}: We rely on a pre-trained class-conditional StyleGAN2 \cite{karras2020training} to generate synthetic CIFAR-10 samples. We create a balanced dataset of 2,500 synthetic images, 250 for each class. Pre-trained StyleGAN models are available at \url{https://github.com/NVlabs/stylegan2-ada-pytorch}.

\textbf{Fake AFHQ Dogs}: Similarly to Fake CIFAR-10, we also create 2,500 synthetic dog images using another pre-trained StyleGAN2 trained on AFHQ Dog images using adaptive discriminator augmentation. Note that since AFHQ does not provide fine-granular labels for dog breeds, images are generated randomly without defining the dog breeds.

\textbf{Permuted}: We generate noisy images by randomly permutating the pixels of the non-members from the CIFAR-10 and Stanford Dog test sets. The resulting images do no longer contain any structural information.

\textbf{Scaled}: We scale samples from the non-members test set after normalization by factor 255. Samples of this dataset follow our theorem and correspond to a scaling factor $\delta=255$.

\subsection{Evaluation Metrics}
For evaluating our experiments precision, recall, false-positive rate (FPR), and mean maximum prediction scores (MMPS) are used. The first three metrics are based on the count of true-positive (TP), false-positive (FP) and false-negative (FN) predictions made by MIAs. The MMPS relies on the prediction score vector $f(x)$ produced by a neural network $f$ given input $x$. $f(x)$ includes the application of a softmax function to compute the prediction scores. We chose precision, recall and FPR since MIAs can be interpreted as a binary classification task. We further computed the MMPS to examine the influence of the maximum prediction scores on MIA classification decisions. The formulas we used to calculate the evaluation metrics can be seen below:
\begin{itemize}
    \item $Precision=\frac{TP}{TP+FP}$
    \item $Recall=\frac{TP}{TP+FN}$
    \item $FPR=\frac{FP}{FP+TN}$
    \item $MMPS=\frac{1}{N} \sum_{n=1}^{N} \max_{i=1,...,d} f(x_n)_i$
\end{itemize}

\newpage

We computed the expected calibration error (ECE) \citep{Naeini2015} on ${K=15}$ bins using the respective test sets. The calibration error is then defined as
\begin{equation}
    ECE = \sum_{i=1}^K \frac{|B_i|}{N} |acc(B_i) - score(B_i)|
\end{equation}
where $|B_i|$ denotes the number of samples in the i-th bin, $N$ is the total number of test samples and $acc$ and $score$ are the accuracy and the mean predicted scores for the true class.

We further define the overconfidence error (OE) \citep{overconfidence_error} as 
\begin{equation}
    OE = \sum_{i=1}^K \frac{|B_i|}{N} \left[ score(B_i) \times \max(score(B_i) - acc(B_i), 0) \right].
\end{equation}
We again used $K=15$ to compute the OE for our experiments.

\section{Additional Experimental Results}\label{app:additional_experimental_results}
We state additional results from our experiments that did not fit into the main part due to page restrictions. 

\subsection{MMPS for Stanford Dogs Models}\label{app:dogs_mmc}
Table \ref{tab:add_mmc_all_attacks} states the mean maximum prediction scores (MMPS) of false-positive and true-negative membership predictions for the default ResNet-50 trained on Stanford Dogs.

\begin{table}[H]
\centering
\resizebox{0.75\columnwidth}{!}{
\begin{tabular}{llcc}
\toprule
\textbf{Dataset}                & \textbf{Attack}   & \textbf{FP MMPS}      & \textbf{TN MMPS} \\
\midrule
\multirow{3}{*}{Stanford Dogs}  & Entropy           & 0.9984                & 0.7565            \\
                                & Max. Score        & 0.9985                & 0.7580            \\
                                & Top-3 Scores      & 0.9979                & 0.7486            \\
\midrule
\multirow{3}{*}{Fake Dogs}      & Entropy           & 0.9977                & 0.7700            \\
                                & Max. Score        & 0.9979                & 0.7724            \\
                                & Top-3 Scores      & 0.9971                & 0.7648            \\
\midrule
\multirow{3}{*}{AFHQ Dogs}      & Entropy           & 0.9978                & 0.7636            \\
                                & Max. Score        & 0.9980                & 0.7661            \\
                                & Top-3 Scores      & 0.9974                & 0.7589            \\
\midrule 
\multirow{3}{*}{AFHQ Cats}      & Entropy           & 0.9972                & 0.7205            \\
                                & Max. Score        & 0.9972                & 0.7208            \\
                                & Top-3 Scores      & 0.9959                & 0.7137            \\
\midrule
\multirow{3}{*}{Permuted}       & Entropy           & 0.9989                & 0.8238            \\
                                & Max. Score        & 0.9990                & 0.8288            \\
                                & Top-3 Scores      & 0.9988 	            & 0.8235            \\
\midrule
\multirow{3}{*}{Scaled}         & Entropy           & 1.0000 	            & 0.8744            \\
                                & Max. Score        & 1.0000                & 0.8744            \\
                                & Top-3 Scores      & 1.0000                & 0.8744            \\
\bottomrule
\end{tabular}
}
\caption{MMPS for false-positive (FP) and true-negative (TN) predictions on the standard ResNet-50 model.}
\label{tab:add_mmc_all_attacks}
\end{table}

\newpage
\subsection{MMPS for Regularized Stanford Dogs Models}\label{app:dogs_regularized_mmc}
Table \ref{tab:add_regularized_dogs} states the mean maximum prediction scores (MMPS) of false-positive and true-negative membership predictions of the top-3 score attack against the various ResNet-50 models trained on Stanford Dogs. 

\begin{table}[h]
\centering
\resizebox{0.8\columnwidth}{!}{
\begin{tabular}{llcc}
\toprule
\textbf{Dataset}                & \textbf{ResNet-50}    & \textbf{FP MMPS}  & \textbf{TN MMPS} \\
\midrule
\multirow{4}{*}{Stanford Dogs}  & Standard              & 0.9979            & 0.7486            \\
                                & Label Smoothing       & 0.9092            & 0.4798            \\
                                & L2 Regularization     & 0.8678            & 0.4331            \\
                                & Temperature           & 0.1349            & 0.0577            \\
\midrule
\multirow{4}{*}{Fake Dogs}      & Standard              & 0.9971            & 0.7648            \\
                                & Label Smoothing       & 0.9036            & 0.4495            \\
                                & L2 Regularization     & 0.8568            & 0.4717            \\
                                & Temperature           & 0.1493            & 0.0775            \\
\midrule
\multirow{4}{*}{AFHQ Dogs}      & Standard              & 0.9974            & 0.7589            \\
                                & Label Smoothing       & 0.9164            & 0.4599            \\
                                & L2 Regularization     & 0.8581            & 0.4648            \\
                                & Temperature           & 0.1457            & 0.0754            \\
\midrule 
\multirow{4}{*}{AFHQ Cats}      & Standard              & 0.9959            & 0.7137            \\
                                & Label Smoothing       & 0.8931            & 0.3517            \\
                                & L2 Regularization     & 0.8414            & 0.3858            \\
                                & Temperature           & 0.0816            & 0.0450            \\
\midrule
\multirow{4}{*}{Permuted}       & Standard              & 0.9988            & 0.8235            \\
                                & Label Smoothing       & 0.9792            & 0.6169            \\
                                & L2 Regularization     & 0.8638 	        & 0.4699            \\
                                & Temperature           & 0.2328            & 0.1171            \\
\midrule
\multirow{4}{*}{Scaled}         & Standard              & 1.0000 	        & 0.8744            \\
                                & Label Smoothing       & 0.9907            & 0.7837            \\
                                & L2 Regularization     & 0.9988            & 0.5668            \\
                                & Temperature           & 0.5806            & 0.9912            \\
\bottomrule
\end{tabular}
}
\caption{MMPS for false-positive (FP) and true-negative predictions (TN)
on the various ResNet-50 models of the Top-3 scores attack. 
}
\label{tab:add_regularized_dogs}
\end{table}

\subsection{Attack Results on ResNet-18}\label{app:resnet_18_add_attack_result}
Table \ref{tab:model_attack_metrics_cifar} states additional training and attack metrics for ResNet-18 trained on CIFAR-10.

\begin{table}[h]
    \centering
    \resizebox{\columnwidth}{!}{
    \begin{tabular}{lc|cc|cc}
    \toprule
                                    & \multicolumn{1}{c}{}  & \multicolumn{2}{c}{\textbf{Calibration}}                                      & \multicolumn{2}{c}{\textbf{Defenses}} \\
    \textbf{ResNet-18}              & \textbf{Standard}     & \textbf{LS}                           & \textbf{LA}                           & \textbf{Temp}                         & \textbf{L2}                           \\
    \midrule
    Train Accuracy                  & 100.00\%              & 100.00\%                              & 100.00\%                              & 100.00\%                              & 68.48\%                               \\
    Test Accuracy                   & 69.38\%               & 71.94\%                               & 69.08\%                               & 69.38\%                               & 58.04\%                               \\
    ECE                             & \textbf{24.02\%}      & $\pmb{\downarrow}$\textbf{13.33\%}    & $\pmb{\downarrow}$\textbf{08.65\%}    & 22.04\%                               & 16.95\%                               \\
    OE                              & \textbf{22.02\%}      & $\pmb{\downarrow}$\textbf{11.20\%}    & $\pmb{\downarrow}$\textbf{00.00\%}    & 00.00\%                               & 12.88\%                               \\
    \midrule
    Entropy Pre                     & 67.35\%               & 79.60\%                               & 68.02\%                               & 63.47\%                               & 51.43\%                               \\
    Entropy Rec                     & 92.32\%               & 94.56\%                               & 53.76\%                               & 84.44\%                               & 28.72\%                               \\
    Entropy FPR                     & \textbf{44.76\%}      & $\pmb{\downarrow}$\textbf{24.24\%}    & $\pmb{\downarrow}$\textbf{25.28\%}    & 48.60\%                               & 27.12\%                               \\
    Entropy AUROC                   & \textbf{76.50\%}      & $\pmb{\uparrow}$\textbf{87.02\%}      & 73.46\%                               & $\pmb{\downarrow}$\textbf{70.86\%}    & $\pmb{\downarrow}$\textbf{50.82\%}    \\
    \midrule
    Max. Score Pre                  & 67.35\%               & 79.72\%                               & 69.36\%                               & 64.96\%                               & 51.40\%                               \\
    Max. Score Rec                  & 92.32\%               & 94.48\%                               & 63.92\%                               & 90.76\%                               & 28.64\%                               \\
    Max. Score FPR                  & \textbf{44.76\%}      & $\pmb{\downarrow}$\textbf{24.04\%}    & $\pmb{\downarrow}$\textbf{28.24\%}    & 48.96\%                               & 27.08\%                               \\
    Max. Score AUROC                & \textbf{77.50\%}      & $\pmb{\uparrow}$\textbf{87.10\%}      & 75.26\%                               & $\pmb{\downarrow}$\textbf{73.61\%}    & $\pmb{\downarrow}$\textbf{51.15\%}    \\
     \midrule
    Top-3 Scores Pre                & 63.84\%               & 76.14\%                               & 68.87\%                               & 67.93\%                               & 50.96\%                               \\
    Top-3 Scores Rec                & 98.04\%               & 99.44\%                               & 67.52\%                               & 91.16\%                               & 43.52\%                               \\
    Top-3 Scores FPR                & \textbf{55.52\%}      & $\pmb{\downarrow}$\textbf{31.16\%}    & $\pmb{\downarrow}$\textbf{30.52\%}    & 43.04\%                               & 41.88\%                               \\
    Top-3 Scores AUROC              & \textbf{77.14\%}      & $\pmb{\uparrow}$\textbf{89.05\%}      & 75.65\%                               & 80.35\%                               & $\pmb{\downarrow}$\textbf{50.64\%}    \\
    \bottomrule
    \end{tabular}
    }
    \caption{Training and attack metrics for ResNet-18 target models trained on CIFAR-10. We compare the results for the standard model to models trained with label smoothing (LS) and Laplace approximation (LA) as calibration techniques and temperature scaling (Temp) and L2 regularization as defense techniques.}
    \label{tab:model_attack_metrics_cifar}
\end{table}

\newpage
\subsection{False-positive rates for CIFAR-10 Models}\label{app:cifar10_fpr_results}
Table \ref{tab:high_fpr_cifar10} states the underlying numerical FPR results for the models trained on CIFAR-10. All three attacks tend to predict unknown samples falsely as members even on data different from the training data distribution. Figure \ref{fig:add_calibration_cifar_results} further plots the FPR for the CIFAR-10 models and the effect of Laplace approximation and label smoothing, respectively.


\begin{figure}[ht!]
    \begin{subfigure}[b]{0.49\linewidth}
        \centering
        \includegraphics[width=\linewidth]{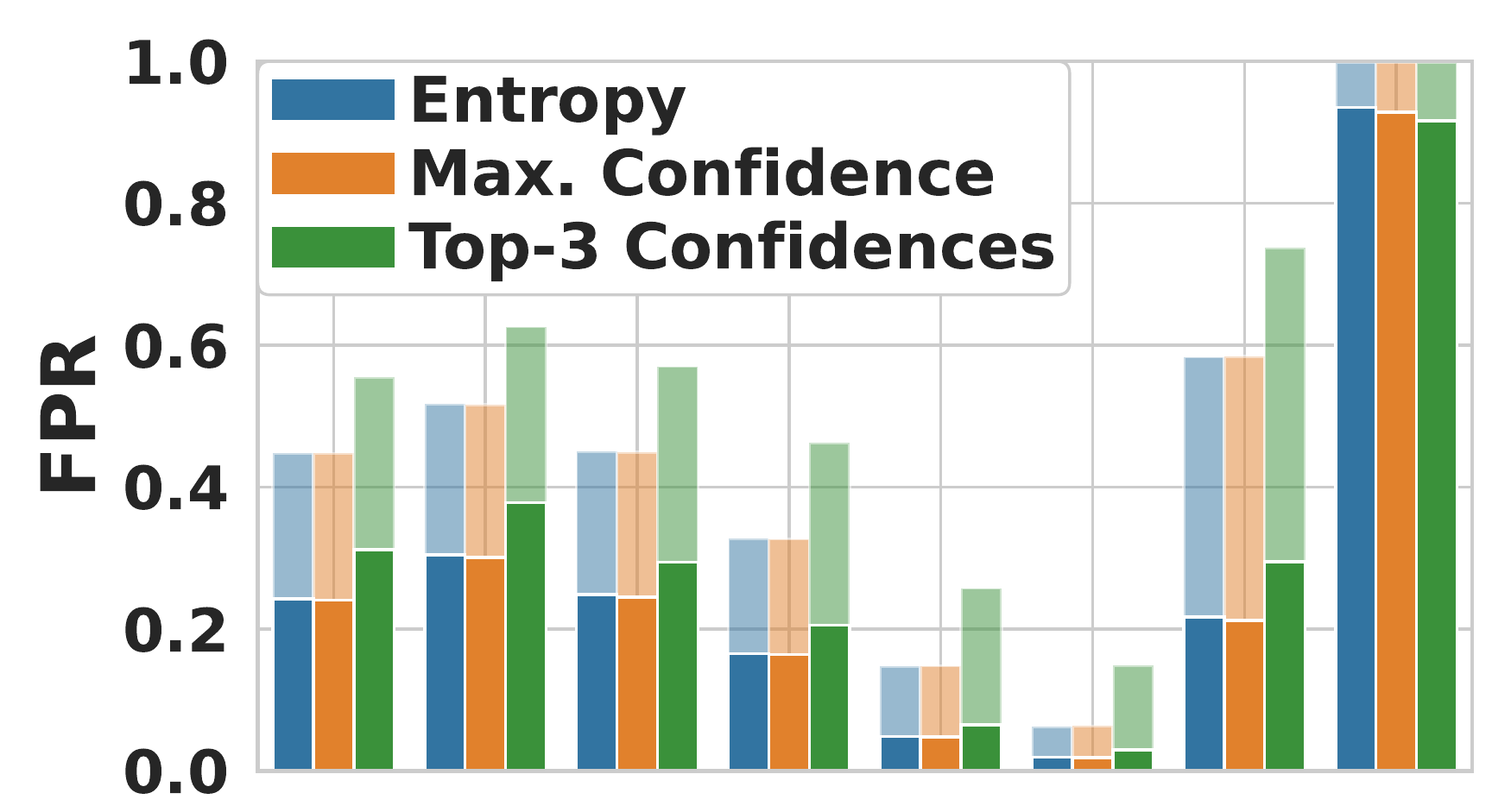}
        \captionsetup{justification=centering}
        \caption{ResNet-18 (LS)}
    \end{subfigure}
    \begin{subfigure}[b]{0.475\linewidth}
        \hfill
        \includegraphics[width=.92\linewidth, height=21.4mm]{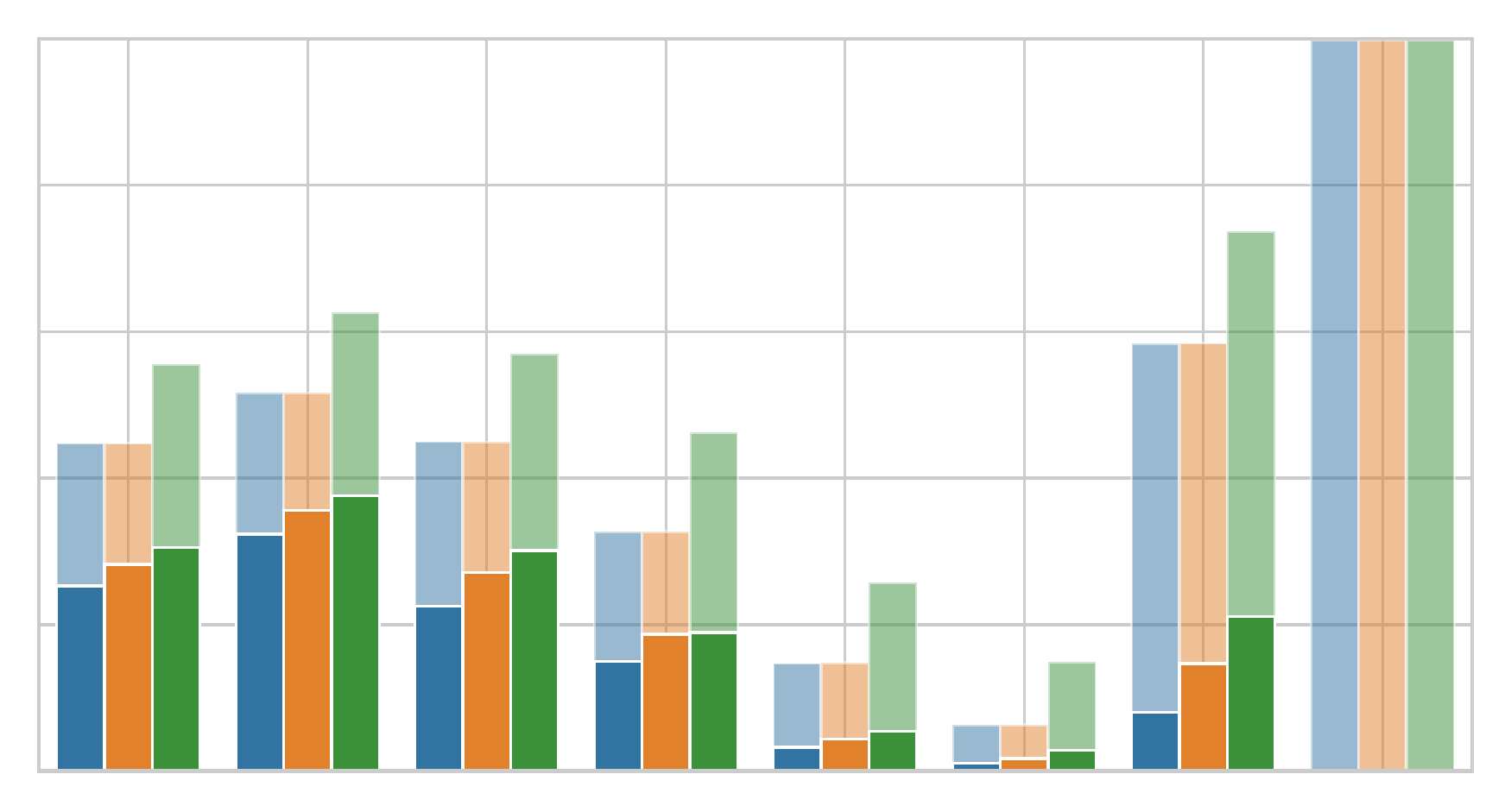}
        \captionsetup{justification=centering}
        \caption{ResNet-18 (LA)}
    \end{subfigure}
    \begin{subfigure}[b]{0.49\linewidth}
        \centering
        \includegraphics[width=\linewidth]{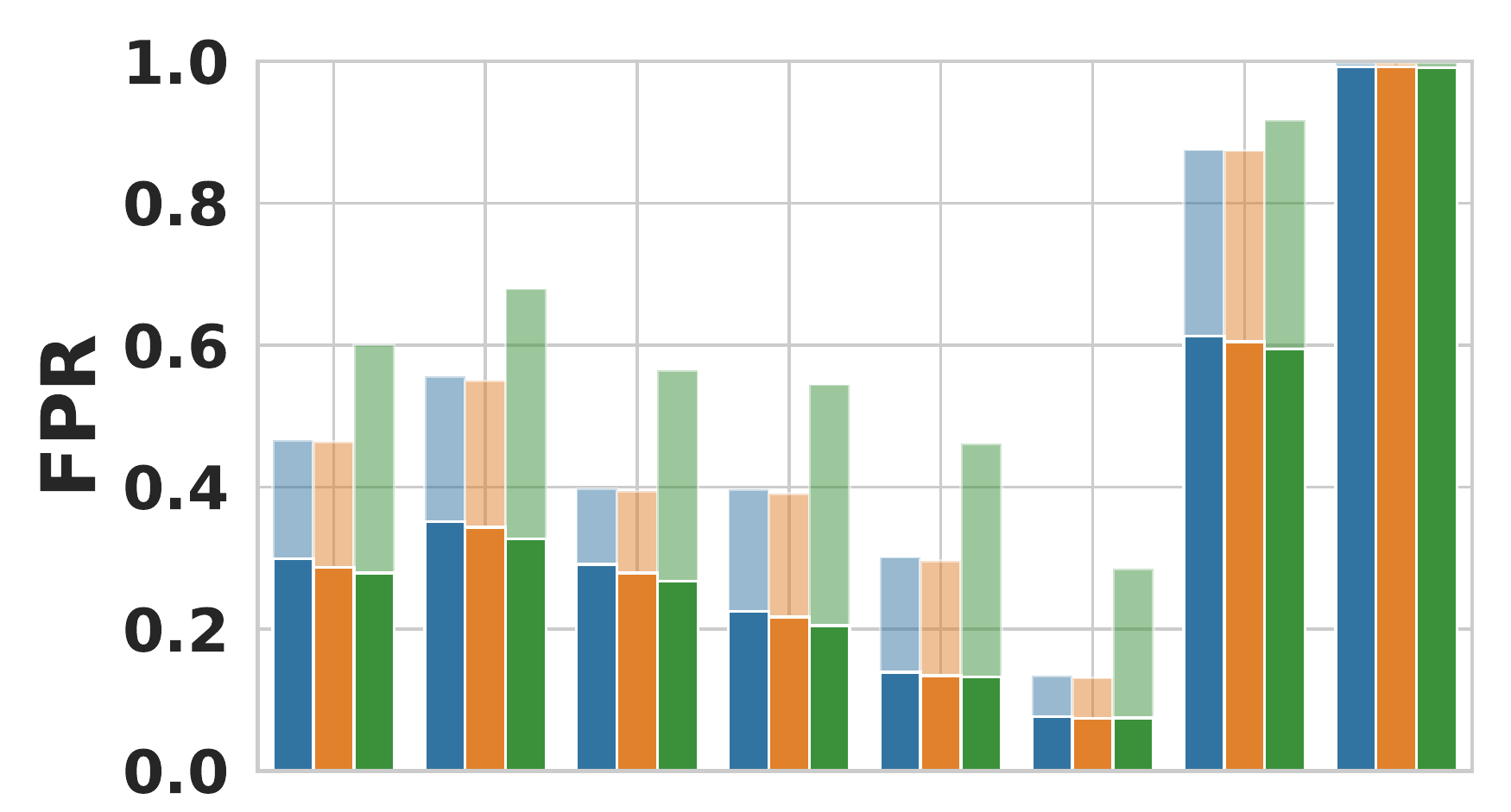}
        \captionsetup{justification=centering}
        \caption{SalemCNN (LS)}
    \end{subfigure}
    \begin{subfigure}[b]{0.475\linewidth}
        \hfill
        \includegraphics[width=.92\linewidth, height=21.4mm]{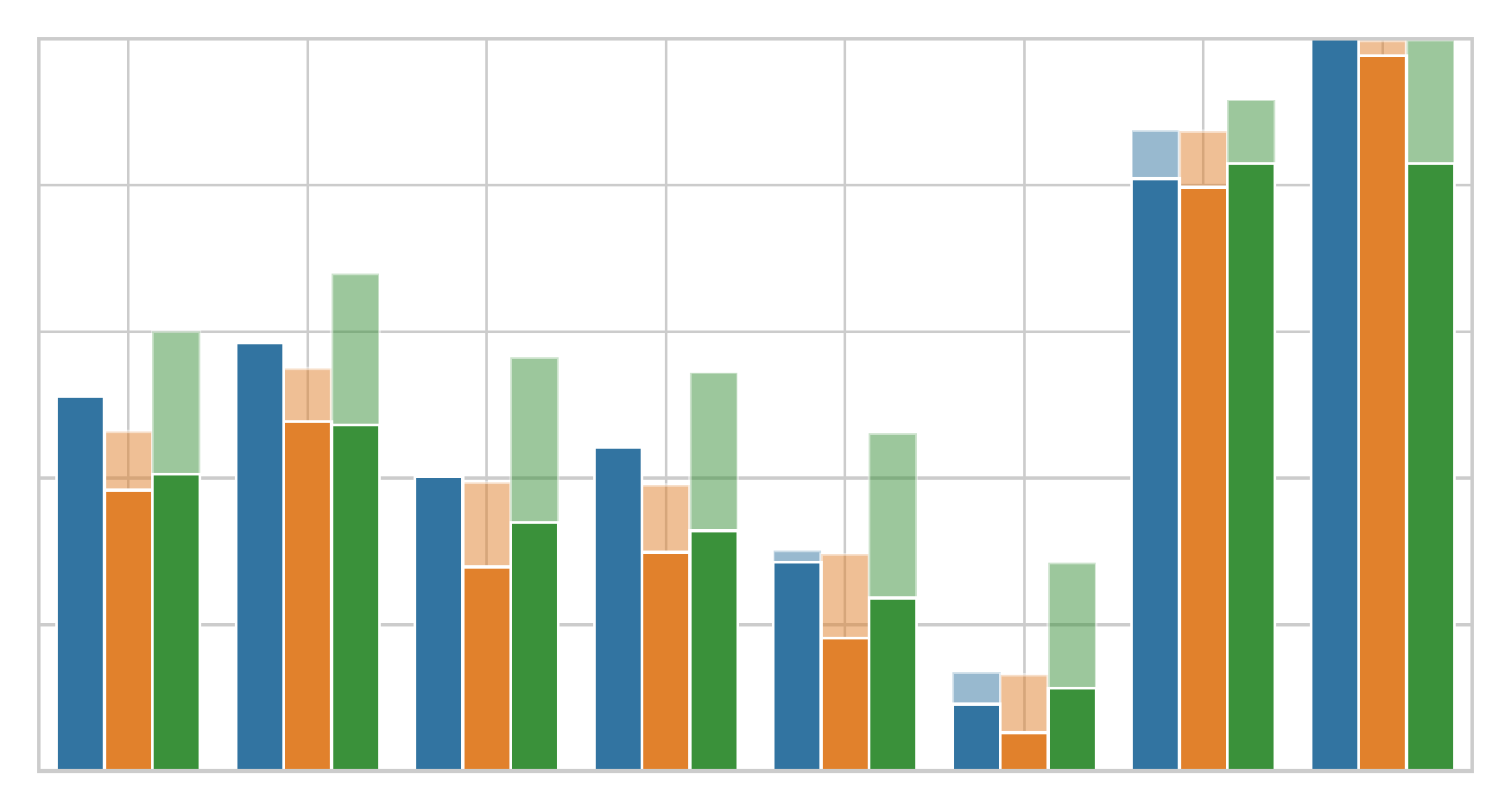}
        \captionsetup{justification=centering}
        \caption{SalemCNN (LA)}
    \end{subfigure}
    \begin{subfigure}[b]{0.49\linewidth}
        \centering
        \includegraphics[width=\linewidth]{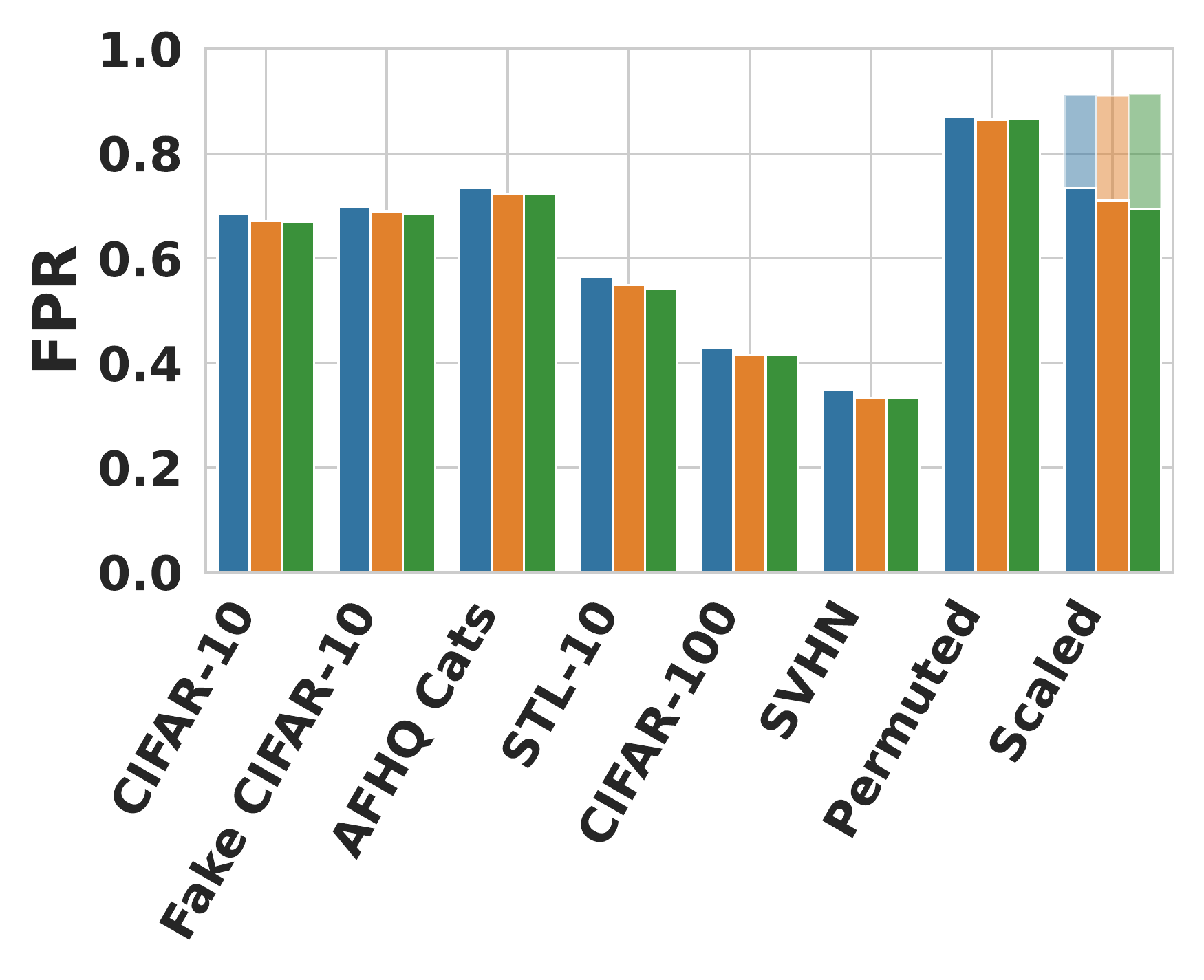}
        \captionsetup{justification=centering}
        \caption{EfficientNetB0 (LS)}
    \end{subfigure}
    \begin{subfigure}[b]{0.48\linewidth}
        \hfill
        \includegraphics[width=\linewidth]{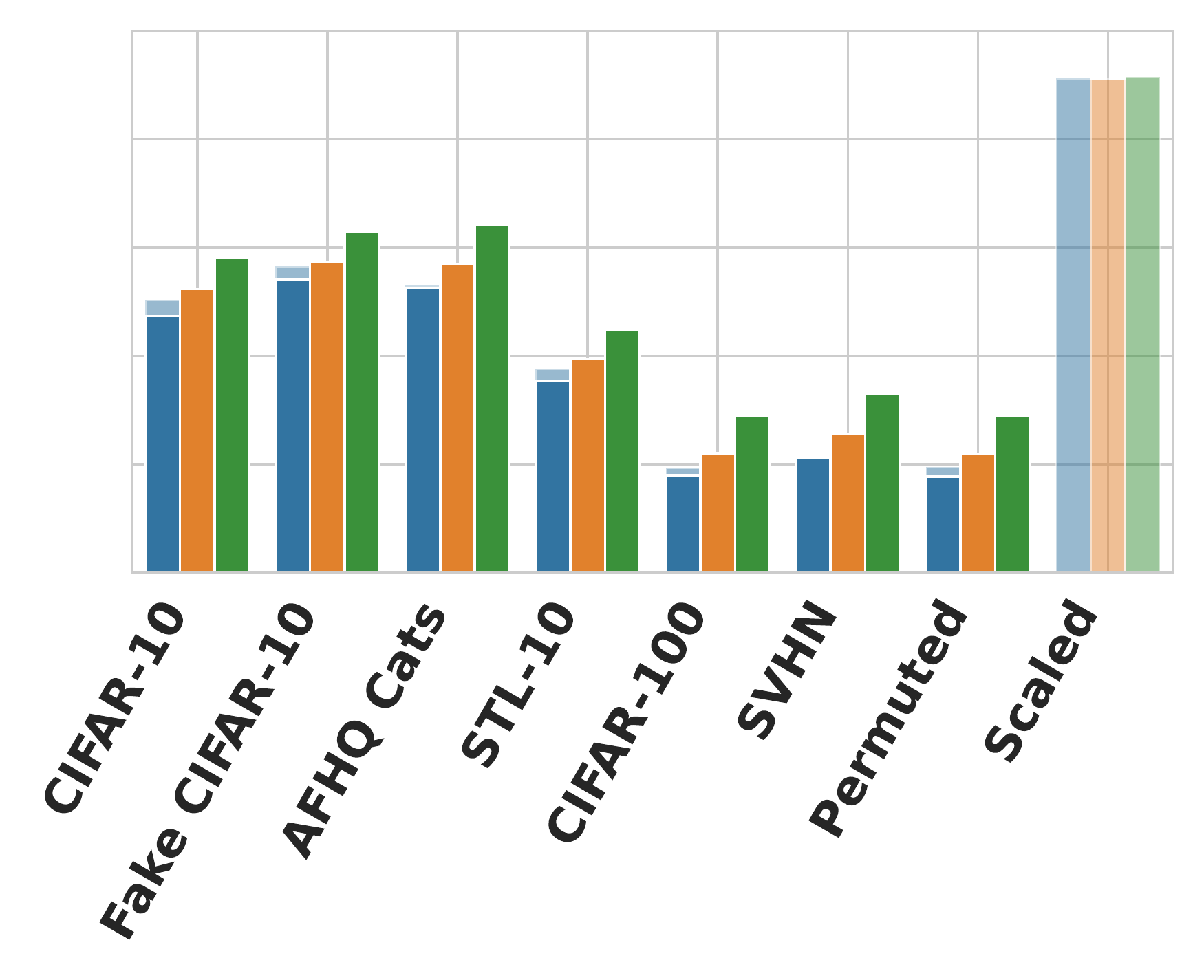}
        \captionsetup{justification=centering}
        \caption{EfficientNetB0 (LA)}
    \end{subfigure}
    \caption{False-positive rates of score-based membership inference attacks against ResNet-18, SalemCNN, and the EfficientNetB0 model trained on the CIFAR-10 dataset. The transparent bars represent the false-positive rate of the standard models while the solid bars represent the false-positive rate of the models with Laplace approximation (LA) or label smoothing (LS), respectively.}
    \label{fig:add_calibration_cifar_results}
\end{figure}

\subsection{MMPS for CIFAR-10 Models}\label{app:cifar_mmc}
Table \ref{tab:mmc_cifar10} states the mean maximum prediction scores (MMPS) of false-positive and true-negative membership predictions for CIFAR-10 models.

\begin{table*}[ht!]
    \centering
    \resizebox{1.75\columnwidth}{!}{
    \begin{tabular}{llcccccccc}
    \toprule
    \textbf{Attack}                     & \textbf{Architecture} & \textbf{CIFAR-10} & \textbf{Fake CIFAR-10}    & \textbf{AFHQ Cats}    & \textbf{STL-10}   & \textbf{CIFAR-100}    & \textbf{SVHN} & \textbf{Permuted} & \textbf{Scaled}   \\
    \midrule
    \multirow{3}{*}{Entropy}            & SalemCNN              & 46.60\%            & 55.60\%                    & 39.76\%                & 39.64\%            & 30.12\%                & 13.44\%        & 87.52\%            & 99.76\%            \\
                                        & ResNet-18             & 44.76\%            & 51.68\%                    & 45.00\%                & 32.72\%            & 14.72\%                & 6.28\%        & 58.36\%            & 99.84\%            \\
                                        & EfficientNetB0        & 50.36\%            & 56.52\%                    & 53.04\%                & 37.64\%            & 19.36\%                & 21.32\%        & 19.48\%            & 91.20\%            \\
    \midrule
    \multirow{3}{*}{Max. Score}         & SalemCNN              & 46.40\%            & 54.96\%                    & 39.40\%                & 39.08\%            & 29.60\%                & 13.12\%        & 87.40\%            & 99.76\%            \\
                                        & ResNet-18             & 44.76\%            & 51.64\%                    & 44.92\%                & 32.68\%            & 14.80\%                & 6.32\%        & 58.44\%            & 99.84\%            \\
                                        & EfficientNetB0        & 50.00\%            & 55.84\%                    & 52.44\%                & 37.28\%            & 19.00\%                & 20.76\%        & 19.24\%            & 91.04\%            \\
    \midrule
    \multirow{3}{*}{Top-3 Scores}       & SalemCNN              & 60.04\%            & 67.96\%                    & 56.48\%                & 54.44\%            & 46.12\%                & 28.48\%        & 91.68\%            & 99.80\%            \\
                                        & ResNet-18             & 55.52\%            & 62.60\%                    & 57.00\%                & 46.28\%            & 25.76\%                & 14.88\%        & 73.72\%            & 99.88\%            \\    
                                        & EfficientNetB0        & 53.40\%            & 58.68\%                    & 56.24\%                & 40.24\%            & 22.76\%                & 24.44\%        & 22.88\%            & 91.48\%            \\
    \bottomrule
    \end{tabular}}
    \caption{False-positive rates (FPR) for score-based membership inference attacks against standard CIFAR-10 target models.}
    \label{tab:high_fpr_cifar10}
\end{table*}

\begin{table*}[ht]
\centering
\resizebox{1.7\columnwidth}{!}{
\begin{tabular}{llcccccc}
\toprule
                                &                       & \multicolumn{2}{c}{\textbf{Entropy}}      & \multicolumn{2}{c}{\textbf{Max. Score}}       & \multicolumn{2}{c}{\textbf{Top-3 Scores}}    \\
                                &                       & \textbf{FP MMPS}      & \textbf{TN MMPS}  & \textbf{FP MMPS}      & \textbf{TN MMPS}      & \textbf{FP MMPS}      & \textbf{TN MMPS}         \\
\midrule
\multirow{3}{*}{CIFAR-10}       & SalemCNN              & 1.0000                & 0.9242            & 1.0000                & 0.9245                & 1.0000                & 0.8987                          \\
                                & ResNet-18             & 1.0000                & 0.8873            & 1.0000                & 0.8873                & 0.9999                & 0.8601                          \\
                                & EfficientNetB0        & 0.9999                & 0.8575            & 0.9999                & 0.8585                & 0.9998                & 0.8483                          \\
\midrule
\multirow{3}{*}{Fake CIFAR-10}  & SalemCNN              & 1.0000                & 0.9212            & 1.0000                & 0.9224                & 1.0000                & 0.8909                          \\
                                & ResNet-18             & 1.0000                & 0.8988            & 1.0000                & 0.8989                & 0.9999                & 0.8694                          \\
                                & EfficientNetB0        & 0.9999                & 0.8649            & 0.9999                & 0.8670                & 0.9998                & 0.8579                          \\
\midrule

\multirow{3}{*}{AFHQ Cats}      & SalemCNN              & 1.0000                & 0.9178            & 1.0000                & 0.9183                & 1.0000                & 0.8862                          \\
                                & ResNet-18             & 1.0000                & 0.9051            & 1.0000                & 0.9053                & 0.9999                & 0.8787                          \\
                                & EfficientNetB0        & 0.9999                & 0.8866            & 0.9999                & 0.8880                & 0.9998                & 0.8784                          \\
\midrule
\multirow{3}{*}{STL-10}         & SalemCNN              & 1.0000                & 0.9144            & 1.0000                & 0.9152                & 1.0000                & 0.8866                          \\
                                & ResNet-18             & 1.0000                & 0.8841            & 1.0000                & 0.8842                & 0.9999                & 0.8549                          \\
                                & EfficientNetB0        & 0.9999                & 0.8428            & 0.9999                & 0.8437                & 0.9998                & 0.8360                          \\
\midrule
\multirow{3}{*}{CIFAR-100}      & SalemCNN              & 1.0000                & 0.9089            & 1.0000                & 0.9095                & 1.0000                & 0.8819                          \\
                                & ResNet-18             & 1.0000                & 0.8560            & 1.0000                & 0.8559                & 0.9999                & 0.8346                          \\
                                & EfficientNetB0        & 0.9997                & 0.8258            & 0.9998                & 0.8266                & 0.9995                & 0.8182                          \\
\midrule
\multirow{3}{*}{SVHN}           & SalemCNN              & 1.0000                & 0.8926            & 1.0000                & 0.8930                & 0.9999                & 0.8700                          \\
                                & ResNet-18             & 1.0000                & 0.8394            & 1.0000                & 0.8393                & 0.9998                & 0.8232                          \\
                                & EfficientNetB0        & 0.9997                & 0.8301            & 0.9998                & 0.8313                & 0.9996                & 0.8231                          \\
\midrule
\multirow{3}{*}{Permuted}       & SalemCNN              & 1.0000                & 0.9405            & 1.0000                & 0.9407                & 1.0000                & 0.9022                          \\
                                & ResNet-18             & 1.0000                & 0.9323            & 1.0000                & 0.9322                & 0.9999                & 0.8956                          \\
                                & EfficientNetB0        & 0.9997                & 0.8086            & 0.9997                & 0.8092                & 0.9995                & 0.7995                          \\
\midrule
\multirow{3}{*}{Scaled}         & SalemCNN              & 1.0000                & 0.8329            & 1.0000                & 0.8329                & 1.0000                & 0.7995                          \\
                                & ResNet-18             & 1.0000                & 0.8926            & 1.0000                & 0.8926                & 1.0000                & 0.8568                          \\
                                & EfficientNetB0        & 1.0000                & 0.8922            & 1.0000                & 0.8941                & 1.0000                & 0.8887                          \\
\bottomrule
\end{tabular}}
\caption{Mean maximum prediction scores (MMPS) for false-positive (FP) and true-negative (TN) member predictions for standard CIFAR-10 models.}
\label{tab:mmc_cifar10}
\end{table*}

\subsection{Additional threshold-free metrics}\label{app:add_metrics}
We state additional threshold-free metrics, namely the area under the precision recall curve (AUPRC) and FPR@95\%TPR, in Table~\ref{tab:add_metrics}.

\begin{table*}[hbt!]
    \centering
    \resizebox{\textwidth}{!}{
    \begin{tabular}{l|c|cc|cc|c|cc|cc}
    \toprule
                                    & \multicolumn{5}{c|}{\textbf{ResNet-18}}                                                                            & \multicolumn{5}{c}{\textbf{ResNet-50}}                                                                             \\
                                    & \multicolumn{1}{c}{}  & \multicolumn{2}{c}{\textbf{Calibration}}  & \multicolumn{2}{c|}{\textbf{Defenses}}& \multicolumn{1}{c}{}  & \multicolumn{2}{c}{\textbf{Calibration}}  & \multicolumn{2}{c}{\textbf{Defenses}} \\
                                    & \textbf{Standard}     & \textbf{LS}       & \textbf{LA}           & \textbf{Temp}     & \textbf{L2}       & \textbf{Standard}     & \textbf{LS}       & \textbf{LA}           & \textbf{Temp}     & \textbf{L2}       \\ 
    \midrule 
    Entr. AUPRC                      & 69.44\%               & 77.42\%           & 65.33\%               & 62.68\%           & 49.95\%           & 71.50\%               & 81.32\%           & 70.96\%               & 57.29\%           & 58.99\%           \\
    Entr. FPR@95\%TPR               & 48.52\%               & 24.64\%           & 61.68\%               & 67.52\%           & 95.04\%           & 60.59\%               & 40.91\%           & 64.92\%               & 86.73\%           & 91.21\%           \\
    Max. AUPRC                       & 79.47\%               & 77.54\%           & 66.79\%               & 64.26\%           & 50.07\%           & 72.68\%               & 81.56\%           & 71.57\%               & 63.76\%           & 59.29\%           \\
    Max. FPR@95\%TPR                & 48.28\%               & 24.56\%           & 54.00\%               & 56.64\%           & 95.44\%           & 60.59\%               & 40.18\%           & 62.05\%               & 74.00\%           & 88.87\%           \\
    Top-3 AUPRC                      & 75.56\%               & 80.89\%           & 66.96\%               & 74.46\%           & 49.92\%           & 72.28\%               & 82.28\%           & 71.94\%               & 75.87\%           & 59.32\%           \\
    Top-3 FPR@95\%TPR               & 48.08\%               & 22.76\%           & 53.80\%               & 48.12\%           & 100.00\%          & 60.40\%               & 40.72\%           & 59.43\%               & 59.23\%           & 88.05\%           \\
    \bottomrule
    \end{tabular}
    }
    \caption{Additional attack results for the CIFAR-10 models (ResNet-18) and the Stanford Dog models (ResNet-50) with various modifications on their respective training dataset.}
    \label{tab:add_metrics}
\end{table*}

\end{document}